\documentclass{article}

\usepackage[preprint]{neurips_2026}

\usepackage[utf8]{inputenc} 
\usepackage[T1]{fontenc}    
\usepackage{hyperref}       
\usepackage{url}            
\usepackage{booktabs}       
\usepackage{amsfonts}       
\usepackage{nicefrac}       
\usepackage{microtype}      
\usepackage{xcolor}         

\usepackage{amsmath}
\usepackage[capitalize,noabbrev]{cleveref}
\usepackage[textsize=tiny]{todonotes}
\usepackage{graphicx}
\usepackage{multirow}
\usepackage{colortbl}
\usepackage{amssymb}
\usepackage{enumitem}
\usepackage{caption}

\usepackage{amsmath, amssymb, amsthm}
\usepackage{mathtools}
\usepackage{inconsolata}
\usepackage{fontawesome5}
\usepackage{colortbl}
\usepackage{pifont}
\usepackage{latexsym}
\usepackage{subfigure}

\usepackage[dvipsnames,table]{xcolor}
\usepackage{graphicx}
\usepackage{pgf}
\usepackage{tcolorbox}

\definecolor{MutedBlue}{RGB}{250, 90, 80}    
\definecolor{MutedSage}{RGB}{158, 172, 148}   
\definecolor{MutedRose}{RGB}{188, 143, 143}   
\definecolor{MutedSand}{RGB}{210, 180, 140}   
\definecolor{MutedSlate}{RGB}{112, 128, 144}  

\definecolor{TestColor}{HTML}{659936}
\definecolor{AgentColor}{HTML}{4249bc}
\definecolor{SystemColor}{HTML}{bc5f42}
\usepackage{booktabs}
\usepackage{multirow}
\usepackage{multicol}
\usepackage{makecell}
\usepackage{array}
\usepackage{longtable}
\usepackage{arydshln} 
\usepackage{titlesec}
\usepackage{subcaption}
\usepackage{enumitem}
\usepackage{microtype}
\usepackage{soul}
\usepackage{xspace}
\usepackage{etoolbox}
\usepackage{wrapfig}

\definecolor{myblue}{RGB}{244, 250, 246}
\definecolor{grayicon}{gray}{0.6}
\definecolor{greenicon}{RGB}{60, 100, 60}
\usepackage{algorithmic}

\definecolor{bluecite}{HTML}{0071BC}
\hypersetup{
    colorlinks,
    citecolor=bluecite,
    linkcolor=red
}
\usepackage[utf8]{inputenc}

\usepackage{cleveref}
\crefname{section}{§}{§§}
\Crefname{section}{§}{§§}

\theoremstyle{plain}
\newtheorem{theorem}{Theorem}[section]
\newtheorem{proposition}[theorem]{Proposition}
\newtheorem{lemma}[theorem]{Lemma}
\newtheorem{corollary}[theorem]{Corollary}
\theoremstyle{definition}

\newtheorem{assumption}[theorem]{Assumption}
\theoremstyle{remark}

\title{Clearer Sight, Fewer Lies: Oriented Pickup Preference Optimization for Multimodal Hallucination Mitigation}

\author{%
Xin Zou$^{1,2}$, Haolin Deng$^{1,2}$, Yibo Yan$^{1,2}$, Shuliang Liu$^{1,2}$, Zhiwei Jin$^{3}$,\\ \textbf{Chen Chen}$^{3}$, \textbf{Haonan Lu}$^{3}$, \textbf{Xuming Hu}$^{1,2}$\thanks{Corresponding author.  <dylan.zoux@gmail.com>} \\ [2.5pt]
  $^{1}$HKUST (GZ)\quad
  $^{2}$HKUST\quad
  $^{3}$OPPO AI Center}
\date{}

\begin{document}

\maketitle

\begin{abstract}
Multimodal Large Language Models (MLLMs) are prone to hallucination as their generation preferences are insufficiently calibrated to visual evidence, causing them to fall back on linguistic priors, rather than faithful grounding. In this work, we start from an empirical observation: when query-relevant visual evidence is explicitly strengthened using the model’s own attention, generation becomes more accurate, suggesting that many failures do not arise solely from missing perception, but from an insufficient tendency to trust the evidence the model has already attended to. Motivated by this finding, we propose Oriented Pickup Preference Optimization (\texttt{OPPO}), an evidence-aware alignment objective that learns preferences over the strength of visual evidence, rather than only response quality. Concretely, \texttt{OPPO} contrasts the same faithful response under stronger, anchored, weaker-evidence views, turning naive visual preference into ordered visual-evidence alignment. We further combine this objective with fine-grained span-level and token-level regularization to stabilize the training. Besides, we provide a theoretical analysis showing that ordered evidence margins induce a positive lower bound on local visual sensitivity. Extensive evaluations across hallucination and general-purpose benchmarks demonstrate that \texttt{OPPO} consistently outperforms baseline methods. 


\end{abstract}




\section{Introduction}
The rapid evolution of MLLMs \citep{Qwen2.5-VL,Qwen3-VL,jin2025andesvl,wang2025internvl3_5,team2026kimi} has fundamentally advanced document intelligence \citep{hu2025mplug,liao2025doclayllm} and complex reasoning tasks \citep{hu2024bliva,vu2025describe}. Yet, their failures still share a common pattern: when visual evidence is weak, atypical, or conflicts with semantic expectations, the model often trusts its internal linguistic prior more than visual evidence, leading to what is known as multimodal hallucination  \cite{huang2025survey}. Fig. \ref{fig:hallcase} illustrates this shared pattern across \textit{generalized} and \textit{text-scene} hallucination. This deficiency poses significant risks in safety-critical \cite{kim2025vru} and text-based VQA tasks \cite{shu2025semantics,huang2025ocr}, where any wrong in object or character recognition can completely undermine the model's reliability \cite{chen2025can}.\vspace{0.2em}

\begin{wrapfigure}{r}{0.5\textwidth}
    \centering\vspace{-1.3em}
    \includegraphics[width=0.99\linewidth]{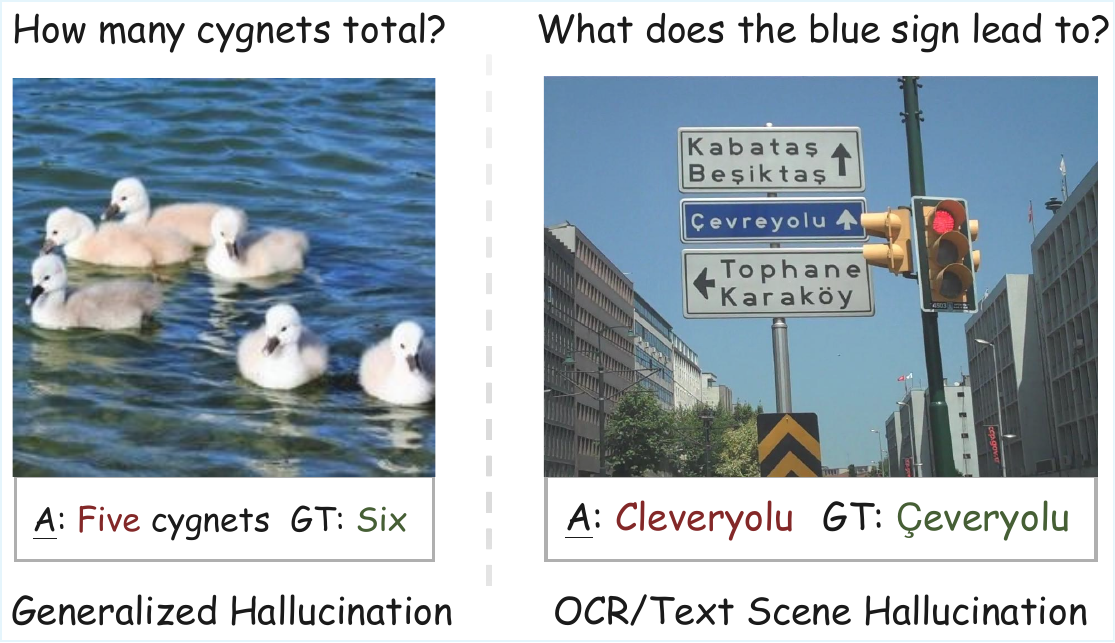}\vspace{-0.5em}
    \caption{Examples of hallucination in MLLMs. }
    \label{fig:hallcase}\vspace{-3em}
\end{wrapfigure}
The root of this misalignment may stem from a fundamental limitation in current post-training paradigms \cite{han2025learning,shu2025large,han2025self}. Previous preference alignment methods typically treat the visual input as a static and passive condition while focusing on optimizing response likelihood \cite{huang2025visual,tang2026revisiting}, thus they may improve response ranking without explicitly enforcing evidence sensitivity. Consequently, when confronted with low-quality or counter-intuitive visual signals, the decoder tends to “short-circuit” the evidence and retreat to its parameterized knowledge \cite{he2025mmboundary}, \textit{i.e.}, the model prefers to predict \textit{“what should be there”} based presumptively on linguistic probability, rather than to decode \textit{“what is actually there”} directly according to task-relevant visual groundedness.

This disconnect raises a pivotal question: \textit{How can we explicitly align the model’s generation preferences with the strength of its visual evidence, rather than merely ranking the final text outputs?} 

Interestingly, our preliminary analysis (\cref{sec:motivation}) reveals a promising paradox, \textit{i.e.}, {even when MLLMs generate hallucinatory content, their cross-modal attention maps often successfully localize the relevant visual regions}. This observation suggests that the failure is not necessarily a lack of “sight”, but rather a policy-level deficiency: the model does not know how to prioritize the evidence it has already attended to. To bridge this gap, we propose oriented pickup preference optimization (\texttt{OPPO}), an evidence-aware alignment paradigm that transforms preference learning from pure \textit{answer ranking} into \textit{evidence-aware ranking}. Specifically, we identify query-relevant regions through attention priors, synthesize stronger and weaker visual views, and train the model to prefer faithful responses more strongly when the supporting evidence is stronger. Compared with conventional DPO-style objectives, \texttt{OPPO} adds an orthogonal supervision signal, \textit{i.e.},  it teaches the model that the confidence assigned to a faithful answer should scale with the strength of the evidence that supports it. To make this supervision more stable and localized, we further introduce span-level and token-level regularization, so that the alignment signal is concentrated on critical content and remains consistent throughout autoregressive generation. Ultimately, \texttt{OPPO} ensures that the model's responses are not just linguistically plausible, but are causally anchored in the visual evidence it perceives.

\begin{figure}[t]
    \centering\vspace{-0.5em}
    \includegraphics[width=1\linewidth, trim={20pt 15pt 44pt 10pt}, clip]{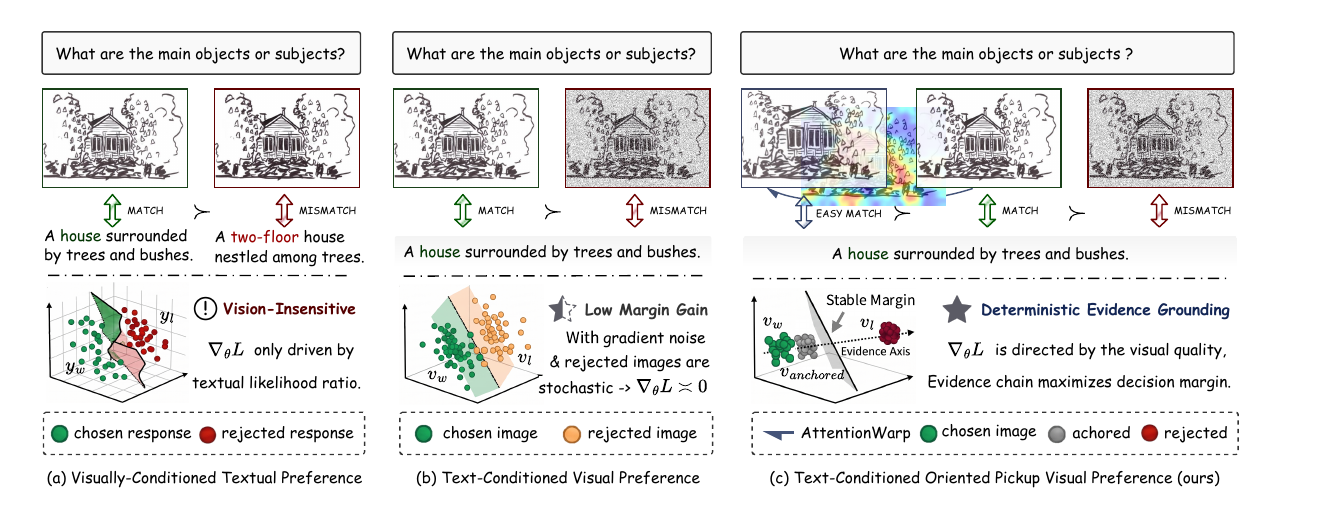}\vspace{-0.25em}
    \caption{Comparison of preference optimization paradigms. (a) DPO focuses on linguistic alignment but is vision-insensitive. (b) Multimodal DPO utilizes stochastic visual augmentation but suffers from high gradient noise and low margin gain. (c) Our method constructs a deterministic visual preference chain to align $\nabla_{\theta}L$ with a query-aware evidence axis, thus maximizing the stable decision margin.}
    \label{fig1:main}\vspace{-1.75em}
\end{figure}
To further clarify the mechanics of this causal anchoring, we contrast our approach with existing preference optimization paradigms in Figure \ref{fig1:main}. Direct Preference Optimization (DPO) operates on the language manifold, but it often yields a decision boundary that is insensitive to fine-grained visual hallucinations. Existing multimodal variants of DPO attempt to bridge this gap by introducing visual perturbations; however, their reliance on random noise or generic cropping often misses the core visual evidence. This leads to a “noisy” latent space where the distinction between positive and negative samples remains ill-defined. In contrast, \texttt{OPPO} explicitly constructs a visual preference chain along a query-aware “evidence axis”. By leveraging attention priors to warp and amplify task-relevant regions, we pull the faithful samples toward high-confidence regions of the latent space while pushing hallucinatory counterparts across the decision hyperplane. This mechanism forces the model to ground its preference scores on explicit visual features, thereby enlarging the optimization margin and effectively mitigating the gradient noise inherent in naive visual preference learning.

Overall, our contributions are summarized as follows: \vspace{-0.25em}
\begin{itemize}[leftmargin=*]
\item We propose to redefine MLLM hallucination as a policy-level distrust of internal visual signals rather than a mere deficiency in visual perception, an insight that allows us to pivot the alignment objective from stochastic linguistic preference to structurally-ordered evidence-aware calibration. 
\item We propose \texttt{OPPO}, which anchors the decision boundary to an ordered visual preference chain, reinforced by span-level regularization and a theoretical lower bound on local visual sensitivity.
\item We demonstrate that \texttt{OPPO} consistently outperforms baseline methods, effectively mitigating hallucinations while preserving general capabilities across both generalized and text-scene benchmarks.
\end{itemize}

\vspace{-0.3em}\section{Preliminary and Motivation} \vspace{-0.2em}
\textbf{Direct Preference Optimization for MLLMs}.
To align MLLMs with human intent, RLHF \cite{bai2022training} and RLAIF \cite{leerlaif} frameworks optimize a policy $\pi_\theta$ to maximize the KL-constrained expected reward:
\begin{equation}\label{eq1}
    \max_{\pi_\theta} \mathbb{E}_{(x,v)\sim\mathcal{D},y\sim\pi_\theta(y|x,v)}[r({x},{y},{v})] - \beta\mathbb{D}_{\mathrm{KL}}[\pi_\theta( \cdot\, |{x},{v}) || \pi_{\mathrm{ref}}( \cdot\, |{x},{v})],
\end{equation}
where $\mathcal{D}$ denotes a preference dataset, ${v}$ is an image, and $\beta$ is the regularization parameter. DPO \cite{rafailov2023direct} bypasses explicit reward modeling by deriving a closed-form solution, the reward is reparameterized:
\begin{equation}\label{eq2}
    r({x}, {y}, {v})=\beta\log\frac{\pi_\theta({y}| {x}, {v})}{\pi_\mathrm{ref}({y}|{x},{v})} + \beta \log Z({x},{v}).
\end{equation}
Here, $Z({x}, {v})$ is a partition function dependent only on the inputs $({x}, {v})$. By integrating Bradley-Terry model \cite{bradley1952rank} and preference pairs $({y}_w, {y}_l)$, the partition function $Z$ cancels out, yielding DPO objective:
\begin{equation}\label{eq3}
    \mathcal{L}_\mathrm{DPO} = -\mathbb{E}_{\mathcal{D}} \left[ \log \sigma \left( \beta \log \frac{\pi_\theta(y_w | x, v)}{\pi_{\text{ref}}(y_w | x, v)} - \beta \log \frac{\pi_\theta(y_l | x, v)}{\pi_{\text{ref}}(y_l | x, v)} \right) \right],
\end{equation}
where $\sigma(\cdot)$ is \texttt{sigmoid} function, which enables direct policy optimization without explicit rewarding.

\textbf{Multimodal Direct Preference Optimization}.
To align MLLMs with visual preferences, multimodal DPO extends the original framework by contrasting visual contexts $(v_w, v_l)$ instead of textual responses. Given the instruction $x$ and response $y$, the objective optimizes the policy $\pi_\theta$ to favor the preferred image $v_w$ over the dispreferred or hallucinatory image $v_l$. The loss function is defined as:
\begin{equation}\mathcal{L}_{\mathrm{MDPO}} = -\mathbb{E}_{\mathcal{D}} \left[ \log \sigma \left( \beta \log \frac{\pi_\theta(y | x, v_w)}{\pi_{\text{ref}}(y | x, v_w)} - \beta \log \frac{\pi_\theta(y | x, v_l)}{\pi_{\text{ref}}(y | x, v_l)} \right) \right].
\end{equation}
By maximizing the likelihood ratio of $v_w$ relative to $v_l$, which compels the model to ground its generations more accurately in the visual input, effectively suppressing the visual hallucination issue.

\vspace{-0.5em}
\subsection{Causes of Multimodal Hallucination and Motivation}\label{sec:motivation}\vspace{-0.4em}
Multimodal hallucinations are often attributed to a variety of factors, including pre-training biases \cite{zhou2023analyzing}, visual uncertainty \cite{leng2023mitigating}, attention sinks \cite{zhang2024redundancy}, and the ``recency bias'' in long autoregressive generation, where the model increasingly relies on previously generated text tokens over the initial visual context \cite{daunhawer2021limitations}. However, we argue that these factors culminate in a common failure mode: a misalignment between the model's internal perception (what it “sees”) and its final decoding policy.

This misalignment is particularly acute in safety-sensitive or text-rich scenarios. Unlike open-ended storytelling, these tasks require visual faithfulness over semantic plausibility. For instance, when identifying a serial code or an object, the model must suppress its linguistic prior in favor of rarer but visually-grounded tokens. We categorize this failure into two types: (a) \textit{Prior-driven overconfidence}, where the language prior dominates the output; and (b) \textit{Evidence under-utilization}, where the model fails to translate valid visual cues into a sufficiently strong generation preference. While existing works often focus on the former by suppressing priors, our work targets the latter, \textit{i.e.}, we want to empower the model to “\textbf{\textit{pick up}}” and trust the visual evidence it has already captured in the encoding.

\begin{figure}[h]
    \centering \vspace{-1.em}
\includegraphics[width=1\linewidth]{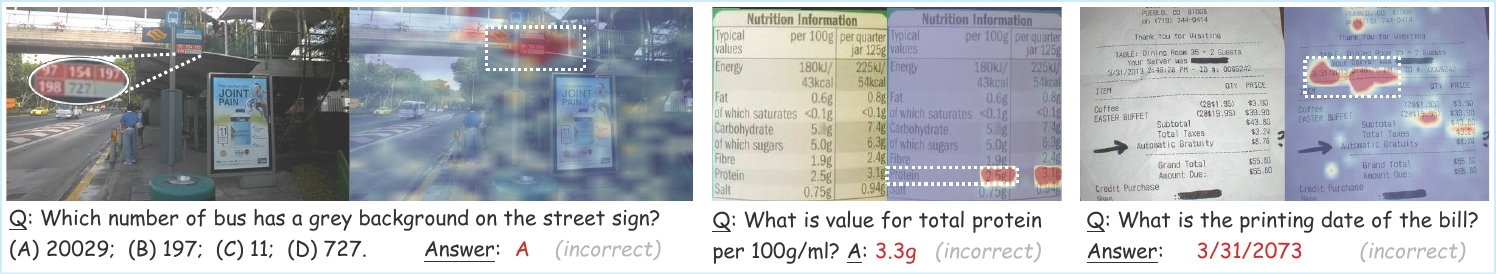}\vspace{-0.6em}
    \caption{Cases of the visual grounding. The attention heatmaps show MLLMs maintain consistent textual-visual attention on query-relevant clues, regardless of the correctness of the final prediction.}
\label{fig:motivation}\vspace{-1em}
\end{figure}

To further validate that the bottleneck lies in \textit{utilization} rather than \textit{perception}, we conducted an empirical analysis on hallucinatory samples. As illustrated in Fig. \ref{fig:motivation}, a striking ``sight-text'' paradox emerges, \textit{i.e.}, MLLMs often still know where to look \cite{zhang2025mllms,jung2025visual}, even when they output ''the wrong answer. For example, when misidentifying a bus number 727' as a more common 720', the model's cross-modal attention heatmap often remains precisely concentrated on the correct numerical region.

\begin{wraptable}{r}{0.48\textwidth}
\centering \footnotesize
  \renewcommand\tabcolsep{1pt}
  \vspace{-0.1em}
  \caption{Results of visual-enhance strategies.}
\label{tab:vicropattwarp}
  \vspace{-0.6em}
    \resizebox{0.486\textwidth}{!}
    {
   \begin{tabular}{l*{1}{>{\centering\arraybackslash}p{4em}}*{1}{>{\centering\arraybackslash}p{3.1em}}*{1}{>{\centering\arraybackslash}p{3.5em}}*{2}{>{\centering\arraybackslash}p{3.3em}}}
    \toprule
    Method\quad & {TextVQA}  & GQA & {DocVQA} & POPE &MMMU \\
    \midrule
    LLaVA-1.5 &{49.3} &60.5  &{18.1}  &85.3 &36.9 \\
    + API \cite{dalal2025constructive}  &{49.9}& 60.6 & {17.4} &85.9 &36.9 \\ 
    + AttCrop\; &{56.3}& 60.9& {22.5}  &87.0 &37.2 \\
    \rowcolor{gray!10} + AttWarp  & 58.1 & 63.7 & 25.5 & 87.5 & 40.4 \\
     \textit{vs.} baseline & {\scriptsize ($\Delta$8.8)} & {\scriptsize ($\Delta$3.2)} & {\scriptsize ($\Delta$7.4)} & {\scriptsize ($\Delta$2.2)} & {\scriptsize ($\Delta$3.5)} \\ \midrule
    Qwen-VL &{81.0} &62.4 &{77.3} & 86.1  &47.3 \\
    + API \cite{dalal2025constructive}  &{81.6} &61.1 &{68.4} &85.8  &47.4  \\  
    + AttCrop\; &{83.8} &60.6 &{82.5} & 86.7   &47.1\\
    \rowcolor{gray!10} + AttWarp  & 84.7 & 64.0 & 84.1 & 87.4 & 50.4 \\
     \textit{vs.} baseline  & {\scriptsize ($\Delta$3.7)} & {\scriptsize ($\Delta$1.6)} & {\scriptsize ($\Delta$6.8)} & {\scriptsize ($\Delta$1.3)} & {\scriptsize ($\Delta$3.1)} \\
    \bottomrule
    \end{tabular}
    }
  \vspace{-1.25em}
\end{wraptable}
Motivated by this observation, we argue that the MLLMs' hallucination is not necessarily ``blindness'' but a policy-level inability to prioritize attended evidence. Then, we conduct a pilot intervention study to verify whether manually amplifying these ``attended but ignored'' visual signals can rectify the model's decision-making. We employ two \textit{visual-enhancement} strategies, cropping and warping, which zoom in on or spatially distort the image based on the model's own attention heatmaps to provide a "magnified" view of the evidence.
As shown in Table \ref{tab:vicropattwarp}, these interventions consistently improve performance, with particularly pronounced gains on hallucination, text-dense, general benchmarks. When visual evidence is stronger, the policy generation naturally shifts towards visual factuality, as \cite{xu2025mitigating,cho2025you,li2025mitigating} demonstrate. 

However, relying on manual image manipulation during inference is computationally expensive and inflexible, which leads to our central hypothesis: \textit{if ``stronger evidence leads to better factuality'', then the alignment objective itself should explicitly encode preferences over evidence strength, rather than merely comparing final text outputs}. By training the model to prefer responses supported by high-confidence visual cues inherently, we can bridge the gap between perception and generation without external intervention and extra overhead, which provides the direct impetus for our method.

\vspace{-0.5em}\section{Oriented Pickup Preference Optimization}\vspace{-0.3em}
OPPO transforms standard response-level alignment into evidence-aware preference optimization. While conventional DPO-style methods focus on which response is superior under a fixed image, we investigate how a model's confidence in a faithful response should scale with evidence strength. By explicitly modeling the relationship between visual clarity and policy preference, we provide a unified framework (Fig. \ref{fig2:oppo}) to mitigate hallucination where linguistic priors override visual grounding.

\begin{figure}[h]
    \centering\vspace{-0.25em}
    \includegraphics[width=1\linewidth, trim={45pt 15pt 120pt 16pt}, clip]{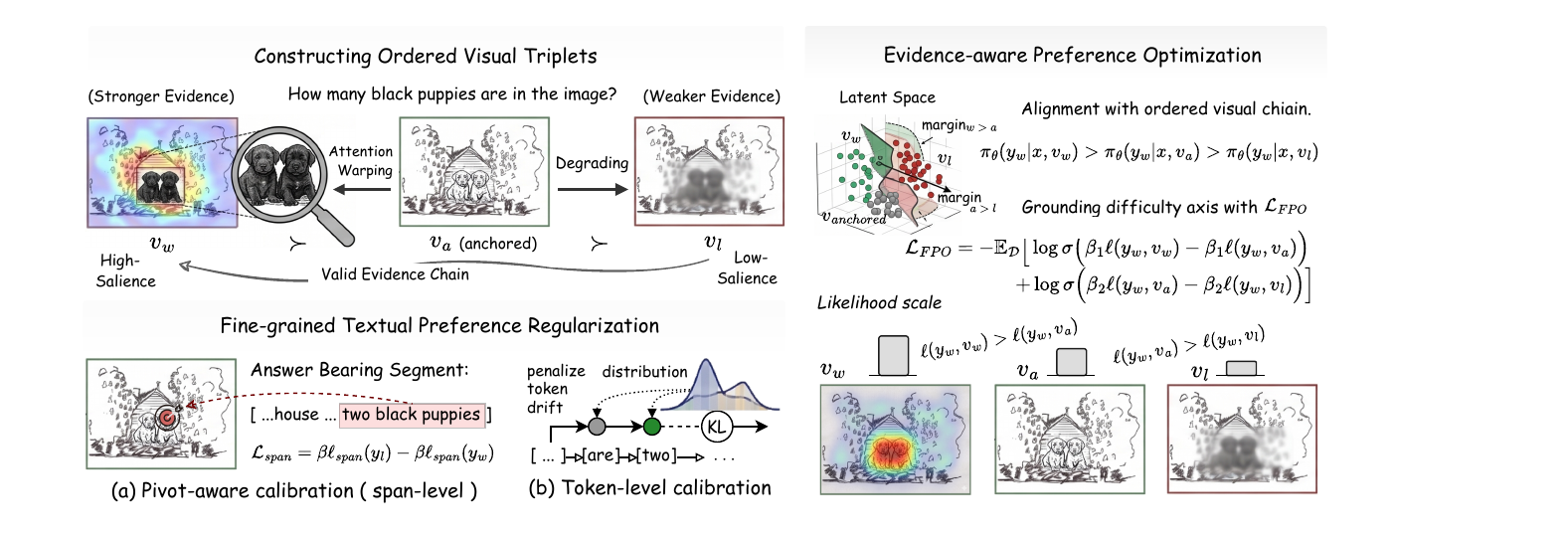}\vspace{-0.2em}
    \caption{Overview of the OPPO framework. We establish an ordered visual triplet $v_w \succ v_a \succ v_l$ to enable evidence-aware preference alignment. By optimizing the FPO objective, the model learns to scale its confidence based on the salience of visual evidence. Further, the framework incorporates span-level and token-level calibration to penalize textual distribution drift via deterministic grounding.}
    \label{fig2:oppo}\vspace{-1.em}
\end{figure}
\subsection{Constructing Ordered Visual Triplets}
To treat evidence strength as a controllable supervisory signal, we construct an ordered triplet of views for each query-answer pair: \textit{stronger-evidence view} $v_w$, \textit{anchored view} $v_a$, and \textit{weaker-evidence view} $v_l$. These views maintain the same semantic target while systematically varying the difficulty of visual grounding. Specifically, we derive $v_w$ by amplifying query-relevant regions identified by the model's internal attention. Given the anchored image $v_a$ and query $x$, we apply a differentiable spatial transformation $\mathcal{T}$ \cite{dalal2025constructive} that magnifies high-salience areas while preserving global context:
\begin{equation}v_w = \mathcal{T}(v_a, \texttt{cross-attention}\,(x, v_a)).
\end{equation}
By presenting a ``clarified'' version of existing evidence without external annotations, $v_w$ forces the model to acknowledge strengthened visual support. $v_l$ is constructed by stochastic degradation in the anchored view $v_a$. In this way, we establish a formal evidentiary hierarchy: $v_w \succ v_a \succ v_l$, which serves as the core supervisory structure, teaching the model to positively use evidence for faithfulness.

\subsection{Evidence-aware Preference Optimization}
With the ordered visual triplet in place, OPPO extends standard DPO in two stages. First, it introduces an evidence-aware preference term that compares the \emph{same faithful response} across different evidence conditions. Second, it adds fine-grained textual regularizers so that this evidence preference is expressed at the answer-bearing span level and throughout the token-generation trajectory.

Naive DPO optimizes whether a chosen response should outrank a rejected one under a fixed image. OPPO adds a complementary question: \emph{should the same faithful response become more preferred when the supporting visual evidence is stronger?}
We answer this with foci-guided preference optimization, which anchors the chosen response $y_w$ to the ordered visual chain introduced above:
\begin{equation}
    \mathcal{L}_{FPO} =-\mathbb{E}_{\mathcal{D}} \big[ \log \sigma \big( \beta_1 \ell(y_w, v_w) - \beta_1 \ell(y_w, v_{a}) \big) + \log \sigma \big( \beta_2 \ell(y_w, v_{a}) - \beta_2 \ell(y_w, v_l) \big) \big],
\end{equation}
where $\ell(y, v) = \log \frac{\pi_\theta(y | x, v)}{\pi_{ref}(y | x, v)}$.
Rather than only learning that $y_w$ should beat $y_l$, it learns that the \emph{same} faithful answer should receive a larger preference margin when the visual evidence is stronger. The first term enforces $v_w \succ v_a$, encouraging the model to rely more heavily on amplified question-relevant evidence. The second term enforces $v_a \succ v_l$, discouraging the model from assigning higher confidence to answers that can be supported mainly by language priors. Together, the cooperation of the two terms makes visual evidence sensitivity part of the preference alignment objective.

\textbf{Fine-grained textual regularization.}
Evidence-level preference alone is still too coarse to shape the generation reliably. In practice, hallucinations are often localized: a single entity, number, or answer span is wrong, while the rest of the response is partially correct. Moreover, even when the final answer looks plausible, token-level drift during autoregressive decoding can reveal that the model is gradually moving away from grounded evidence, and easily triggering reward hacking \cite{eisenstein2023helping,miao2024inform}. To address these, we introduce decomposed optimization to facilitate fine-grained reward attribution:

\textbf{\textit{(a) Pivot-aware calibration.}} For critical span $y_s$, we think it should be assigned more reward. Following \cite{yu2024rlhf,fuchip}, we perform a sequence-level comparison between the chosen $y_w$ and rejected $y_l$, identify the critical span $y_s$, i.e.,  the answer-bearing segment where the chosen and rejected trajectories diverge most strongly. Then, boost its contribution via a weighted log-probability term:
\begin{equation}
\mathcal{L}_{span} = -\mathbb{E} \big[ \log \sigma \big( \beta \ell_{span}(y_w, v_{a}) -  \beta\ell_{span}(y_l, v_{a}) \big) \big] ,
\end{equation}
where $\ell_{span}(y, v_{a}) = \frac{1}{s} \big( \ell(y, v_{a}) + \sum\nolimits_{y_i \in y} \ell(y_{<i}, v_{a})\big)$, and $s= |y| + |y_s|$ serves as the length-normalization factor. By explicitly calibrating the pivot span, the term improves credit assignment.

\textbf{\textit{(b) Token-level calibration.}} As the autoregressive process is token by token, we think alignment on the token level is natural and makes MLLM keep diversity \cite{zeng2024token}. To further construct a fine-grained constraint, different from only response-level preference optimization (Eq. \ref{eq3}) in the previous methods, we enforce token-level supervision for consistency across the full generation trajectory:
\begin{equation} 
\mathcal{L}_{token} = \mathbb{E}_{(y_w,y_l,x,v_{a})\sim\mathcal{D}} \big[ \beta [\tau(y_w, v_{a})]_{\bot}  - \beta \tau(y_l, v_{a}) \big] ,
\end{equation}
where $[\cdot]_{\bot}$ denotes stop-gradient operation, $\tau(y, v_{a}=\sum_t \mathbb{D}_{\mathrm{KL}}[\pi_\theta(\cdot|x,v_a,y_{<t})\|\pi_{ref}(\cdot|x,v_a,y_{<t})]$ aggregates the sequential KL divergence along the generation path. Specifically, this token-level calibration penalizes accumulative per-token drift and keeps the aligned policy close to a stable reference trajectory during alignment training, which is especially important in complex objects or text-rich scenarios, where a locally wrong token can invalidate the entire generation.

The final \texttt{OPPO} objective establishes a comprehensive optimization framework by integrating standard response preference with evidence-oriented preference and a system of fine-grained textual regularization that constrains the model's generation space:
\begin{equation}
\mathcal{L}_{OPPO}= \mathcal{L}_{DPO} +  \mathcal{L}_{DePO} + \gamma_1  \mathcal{L}_{FPO} ,
\end{equation}
where $\mathcal{L}_{DePO}=\mathcal{L}_{span} + \gamma_2*\mathcal{L}_{token}$, and $\gamma_1,\gamma_2$ balance the influence of evidence-aware preference learning and fine-grained textual regularization. Conceptually, standard $\mathcal{L}_{DPO}$ teaches the model \emph{which answer} to consistently prefer, $\mathcal{L}_{FPO}$ teaches it \emph{how that preference should scale with visual evidence strength}, and $\mathcal{L}_{DePO}$ teaches it \emph{where and how} to express that preference during generation.

\section{Experiments}\label{sec:exp}
    \vspace{-0.5em}
\subsection{Experimental Settings}
\noindent
\textbf{Training Data.}
We choose to mix the part of RLAIF-V~\cite{yu2024rlaif} and TextVQA-train \cite{singh2019towards} datasets with total 8000 training samples as our training data, where for RLAIF-V dataset, as it has preference pairs we directly randomly sample a subset, and for the TextVQA-train data, we same first random sample a subset, then use the RLAIF-V method constructs chosen and rejected preference pairs. Accordingly, we prepare the correct responses with both descriptive ($e.g.$, \texttt{Describe
the image in detail}) and non-descriptive ($e.g.$, \texttt{<Question>, Please answer in
one word}) instructions.

\begin{table*}[t]\vspace{-0.5em}
    \caption{Quantitative evaluation of hallucination performance. Hal. means hallucination rate, lower is better for \textcolor{greenicon}{$\downarrow$} marked metrics. \textbf{Bold} and \underline{underlined} denote the best and the second value, respectively.}
    \label{tb:sota}
    \vspace{-0.4em}
    \resizebox{\textwidth}{!}
    {
        \begin{tabular}{*{1}{>{\arraybackslash}p{6.em}}*{8}{>{\centering\arraybackslash}p{2.35em}}*{2}{>{\centering\arraybackslash}p{2.6em}}*{2}{>{\centering\arraybackslash}p{2.4em}}*{1}{>{\centering\arraybackslash}p{2.1em}}}
        \toprule
        \multirow{2}{*}{Methods} & 
        \multicolumn{2}{c}{{POPE}} &
        \multicolumn{2}{c}{{MMHal}} & 
        \multicolumn{4}{c}{{AMBER (Generative)}} & 
        \multicolumn{2}{c}{{OBJHal}} &
        \multicolumn{2}{c}{{HallusionBench}} &
        \multirow{2}{*}{{\begin{tabular}{l}\hspace{-0.17cm}Avg. \\\hspace{-0.17cm}Hal. $\downarrow$\end{tabular}}} \\
        \cmidrule[0.5pt](lr){2-3} \cmidrule[0.5pt](lr){4-5} \cmidrule[0.5pt](lr){6-9} \cmidrule[0.5pt](lr){10-11}\cmidrule[0.5pt](lr){12-13}
        & {Acc\textcolor{greenicon}{$\uparrow$}} & {Pre\textcolor{greenicon}{$\uparrow$}} & {Score\textcolor{greenicon}{$\uparrow$}} & {Hal.\textcolor{greenicon}{$\downarrow$}} & {CHAIR\textcolor{greenicon}{$\downarrow$}} & {Cover\textcolor{greenicon}{$\uparrow$}} & {Hal.\textcolor{greenicon}{$\downarrow$}} & {Cog\textcolor{greenicon}{$\downarrow$}} &{CHAIR$_\text{s}$\textcolor{greenicon}{$\downarrow$}} &  {CHAIR$_\text{i}$\textcolor{greenicon}{$\downarrow$}} & $\;$qAcc\textcolor{greenicon}{$\uparrow$}\hspace{-0.1cm} & fAcc\textcolor{greenicon}{$\uparrow$}& \\ 
        \midrule
        {MemVR}\texttt{\scriptsize{(ICML25)}}
        &85.1&89.1&2.33&54.3&8.6&50.3&39.7&4.6&48.6&16.8 & 5.2 & 11.8 & 47.0\\     
        {HALVA}\;\texttt{\scriptsize{(ICLR25)}}
        &87.2&78.9&2.12&59.1&6.9&52.8&33.2&3.5&47.3&14.6 & 21.4 & 25.1 & 46.2\\     
        {RLAIFV}\texttt{\scriptsize{(CVPR25)}} 
        &88.1&88.0&{2.89}&42.8&3.0&50.3&16.5&1.0&13.7&4.2 & 26.9 & 30.1 & 29.7\\
        \midrule
        {LLaVA-1.5}
        &84.9&89.1&2.18&59.2&8.8&50.1&40.4&4.7&54.7 & 26.5 &3.9 &11.5 &49.8\\
        + {DPO}
        &87.6&79.8&2.14&65.8&6.5&55.5&34.5&2.3&58.1&7.2 & 7.1 & 7.6 &50.1\\
        + {mDPO}
        &87.8&82.0&\underline{2.39}&65.2&4.4&52.4&24.5&2.4&35.7 & 9.8 &6.8 & 9.5 &44.9\\
        + {CHiP}\texttt{\scriptsize{(ICLR25)}}
        &\textbf{88.1}&\textbf{92.2}&2.32&\underline{56.2}&\underline{2.9}&\underline{57.3}&\underline{19.9}&\underline{1.9}&\textbf{25.3}&\textbf{6.2} & \textbf{7.3} & \underline{9.8} &\underline{38.1} \\
         \rowcolor{gray!10}  + {OPPO}(ours)
        &87.9&\underline{91.9}&\textbf{2.54}&\textbf{52.0}&\textbf{2.3}&\textbf{59.3}&\textbf{14.9}&\textbf{1.8}&\underline{28.9}&\underline{6.5} & \underline{7.2} & \textbf{10.1} &\textbf{33.5} \\ \midrule
        Qwen2.5-VL
        &86.3&85.8&3.29&27.5&4.6&54.6&21.1&1.3&40.7&8.6 & 39.6 & 34.9 &24.3\\
        + DPO  & \underline{86.8}&86.2    & 3.31&27.0 & 4.4&\underline{54.9}&20.5&1.3 & 38.5&8.2  & \underline{39.8}&35.1  &23.7   \\
        + mDPO & 86.6&85.9    & 3.30&27.2 & 4.5&54.7&20.8&1.3 & 39.2&8.4  & 39.7&35.0  &24.0    \\
        + {CHiP}\texttt{\scriptsize{(ICLR25)}} &86.5&\underline{86.8}&\underline{3.32}&\underline{26.5}&\underline{4.2}&54.8&\underline{20.3}&\underline{1.3}&\underline{37.2}& \underline{7.9} & \underline{39.8} & \underline{35.8} & \underline{23.4}\\
        \rowcolor{gray!10} + {OPPO}(ours) & \textbf{87.0} &87.1 & \textbf{3.38}&25.4  & \textbf{4.0}& \textbf{55.8}&19.2 &\textbf{1.2}  & \textbf{35.8} &\textbf{7.7} & \textbf{40.2} &\textbf{36.5} &\textbf{22.3} \\ 
        \midrule
        {LLaVA-NeXT}
        &85.5&90.3&3.50&40.6&8.7&61.1&49.7&4.2&11.3& 6.6 &4.4 &8.9 & 45.2\\
        + {DPO}
        &85.2&84.3&3.41&45.8&6.2&\textbf{61.2}&38.3&3.1&9.0 & 5.9 &\textbf{5.7} & 12.1 & 42.1\\
        + {mDPO}
        &87.8&82.0&3.49&40.6&4.2&57.9&28.6&\textbf{1.8}&8.7 & 5.4 & 4.8 &10.4 &34.6\\
        + {CHiP}\texttt{\scriptsize{(ICLR25)}} 
        &\textbf{88.0}&\underline{91.8}&\underline{3.51}&\underline{39.6}&\underline{4.0}&57.6&\underline{27.2}&\underline{1.9}&\underline{5.4}& \underline{4.2} & 5.3 & \underline{12.7} &\underline{33.5} \\
        \rowcolor{gray!10} + {OPPO}(ours)  &\underline{87.9}&\textbf{92.5}&\textbf{3.54}&\textbf{39.5}&\textbf{3.9}&\underline{58.0}&\textbf{26.4}&\underline{1.9}&\textbf{4.3} & \textbf{3.0} &\underline{5.6} &\textbf{13.0} &\textbf{32.8} \\
        \bottomrule
        \end{tabular}
    }
    \vspace{-1em}
\end{table*}

\noindent\textbf{Datasets and Evaluation.} 
To rigorously assess the effectiveness of our proposed method, we conduct a comprehensive set of experiments across OBJHal (CHAIR) benchmark~\cite{rohrbach2018object}, POPE~\cite{li2023evaluating}, MMHal-Bench~\cite{sun2023aligning}, HallusionBench~\cite{guan2024hallusionbench}, AMBER~\cite{wang2023llm}, TextHalu-Bench~\cite{shu2025semantics}, KIE-HVQA~\cite{he2025seeing}, and general benchmark MMBench\cite{liu2024mmbench}, MMMU \cite{hendrycks2020measuring}, LLaVA-Wild \footnote{https://huggingface.co/datasets/lmms-lab/llava-bench-in-the-wild}.
More details are in the Appendix \ref{apx:benchmarks}.\vspace{0.25em}

\noindent
\textbf{Implementation Details.}  We train LLaVA-1.5-7B, Qwen2.5-VL-7B, and LLaVA-Next-8B with a learning rate of 5e-7 and a batch size of 32, and take about two hours on 4 H20 GPUs per epoch, for a total of 2 epochs. All settings of baseline methods follow the default configurations from the original papers. We set warmup ratio to 0.03, and set $\beta=0.5$, $\gamma_2=0.1$. The enhanced images are offline built based on the original image of the forward attentional warping process, for our 8K samples take few hours. The rejected images are built based on the original input image of the forward diffusion process at 500 steps. For fairness, all compared preference-optimization baselines are re-evaluated in the same pipeline with greedy decoding (temperature $=0$), and the same judge model is employed whenever a benchmark requires model-based evaluation.  More details are in Appendix \ref{apx:implement}. 

\noindent
\textbf{Baselines.}
We primarily compare \texttt{OPPO} with standard DPO, mDPO, and CHiP \cite{fuchip}. 
While other MLLMs cannot be directly compared due to differences in base models, preference data, and alignment methods, we provide MemVR \cite{zoulook}, RLAIF-V \cite{yu2024rlaif}, HALVA \cite{sarkar2024halva}'s results for reference, where MemVR is a representative efficient training-free paradigm, RLAIF-V, HALVA need training.

    \vspace{-0.5em}
\subsection{ Experimental Results}
    \vspace{-0.25em}
We evaluate \texttt{OPPO} as a general-purpose, evidence-aware alignment objective across generalized hallucination, text-scene hallucination, and general benchmarks. Results are listed in Table~\ref{tb:sota}, \ref{tab3:text}, \ref{tab4:general}.

\begin{figure}[t]
\addtolength{\subfigcapskip}{-5pt}
  \centering\vspace{-1em}
        \subfigure[Vanilla (LLaVA-NeXT)] {\hspace{-0.4cm} 
        \includegraphics[width=0.255\columnwidth]{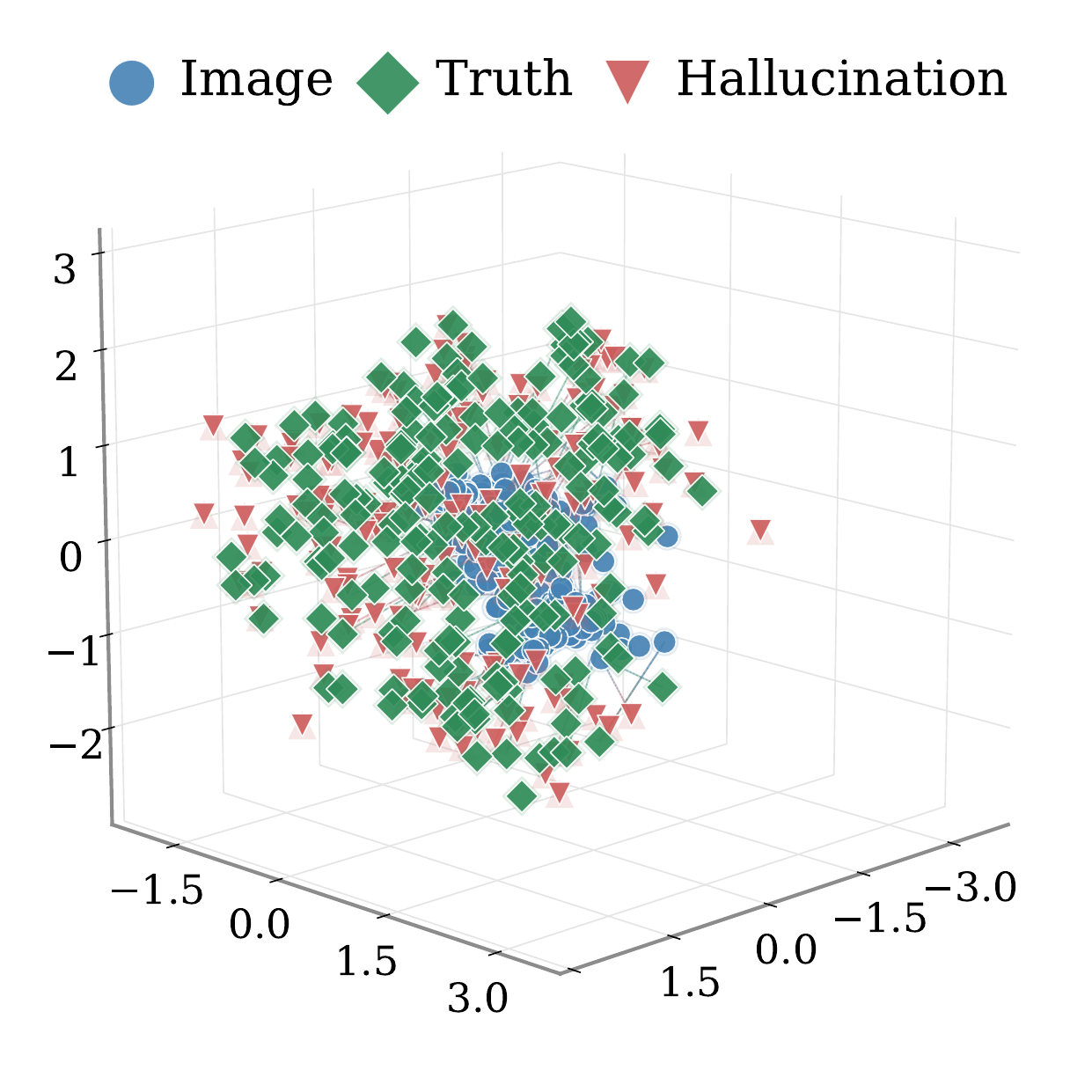}
        }\hspace{-0.25cm}
        \subfigure[Vanilla with mDPO] { 
        \includegraphics[width=0.255\columnwidth]{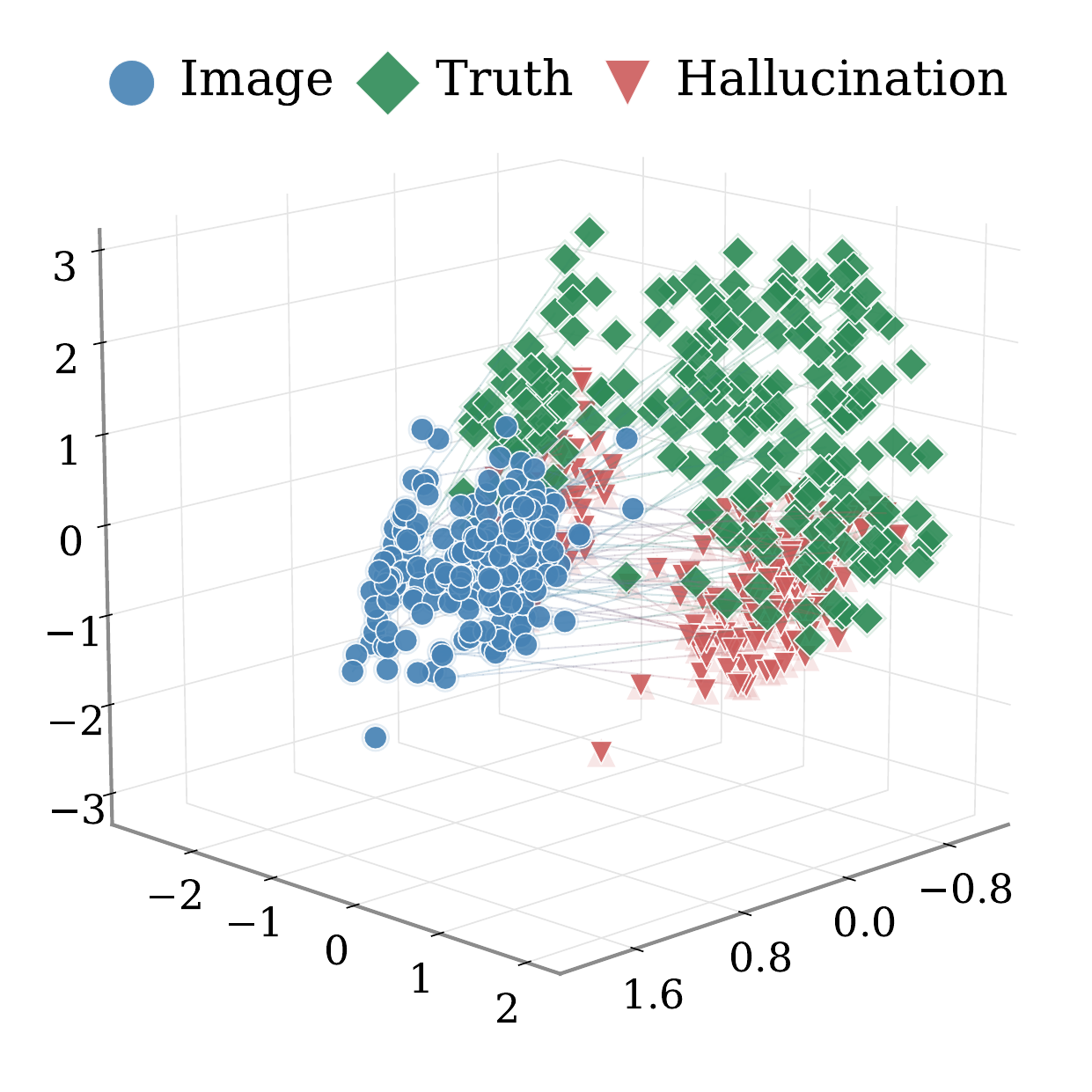}
        }\hspace{-0.25cm}
        \subfigure[Vanilla with CHiP] { 
        \includegraphics[width=0.255\columnwidth]{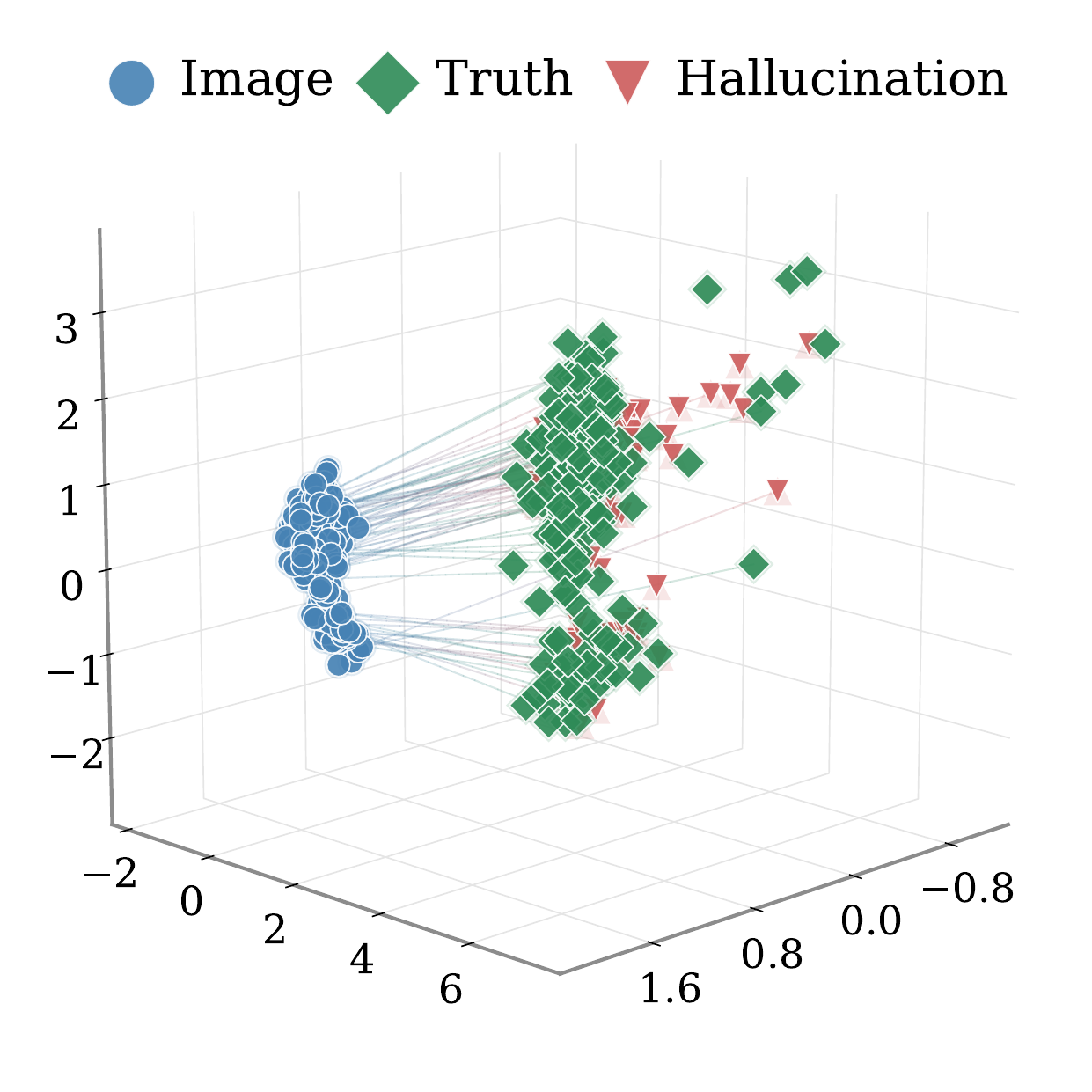}
        }\hspace{-0.25cm}
        \subfigure[Vanilla with OPPO (ours)] { 
        \includegraphics[width=0.255\columnwidth]{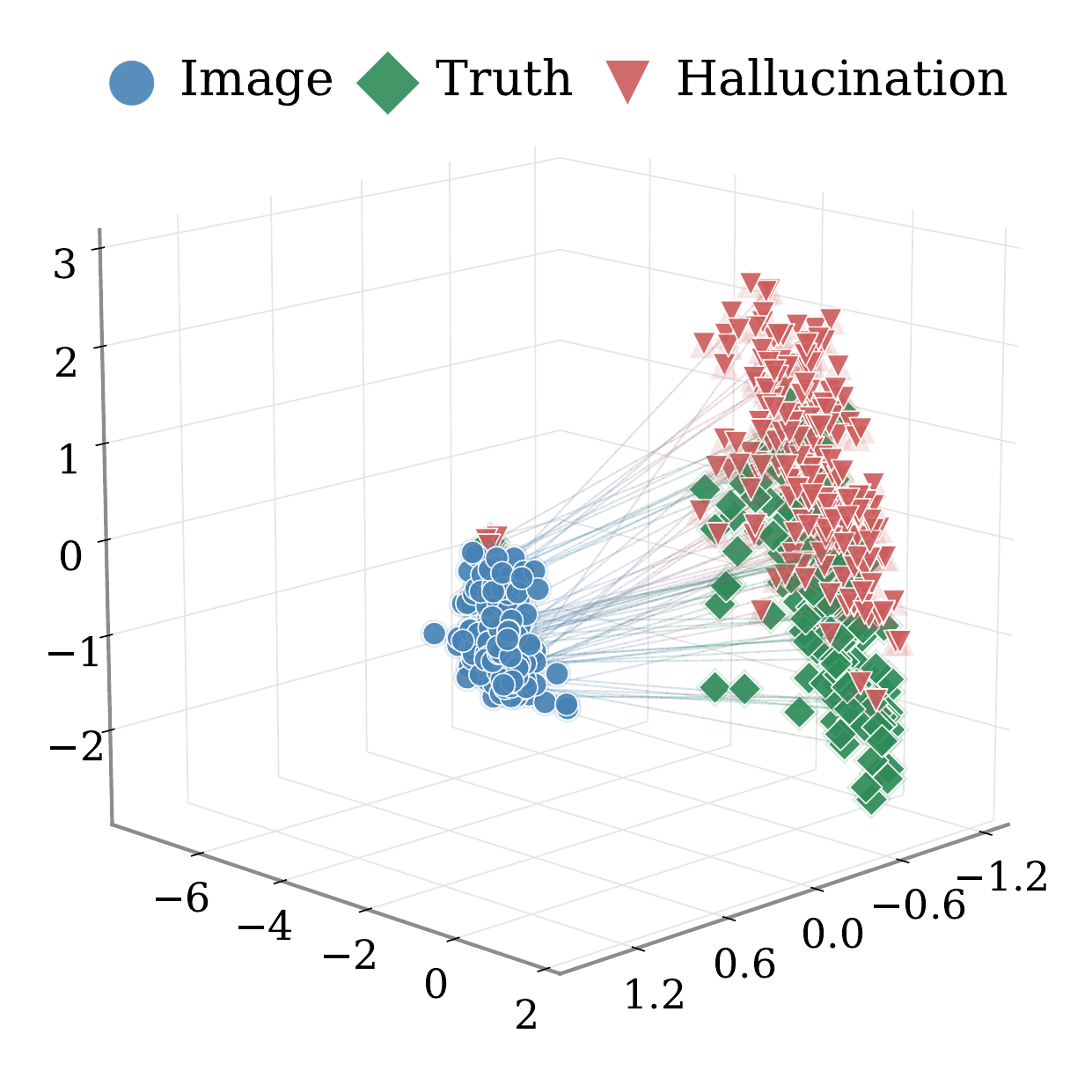}\hspace{-0.3cm}\vspace{-5em}
        }\vspace{-0.75em}
\caption{PCA visualization of representation distributions on AMBER dataset (\textit{200 samples}). (a) Vanilla: significant overlap between image, hallucination, and truth embeddings. (b) DPO: separates truth from hallucination but fails to distinguish image features. (c) CHiP: distinguishes images but shows poor separation between truth and hallucination. (d) OPPO: yields the discriminative clusters.}
\label{fig:pca} \vspace{-1.5em}
\end{figure}
\textbf{Generalized Hallucination}. 
Under matched-backbone comparisons, Table \ref{tb:sota} shows that OPPO delivers consistent and competitive gains over strong preference-optimization baselines. On LLaVA-1.5 \cite{liu2024llava}, Qwen2.5VL \cite{Qwen2.5-VL}, LLaVA-NeXT \cite{liu2024llavanext}, OPPO achieves the best performance, while also improving key grounding-sensitive metrics such as MMHal score, AMBER hallucination rate, and OBJHal CHAIR, which illustrate learning preferences of visual evidence strength is effective on generalized hallucination settings.
Beyond average scores, OPPO exhibits favorable object-level grounding behavior. While methods like DPO and mDPO often trade precision for recall on OBJHal, OPPO reduces both $\text{CHAIR}_s$ and $\text{CHAIR}_i$, especially on LLaVA-NeXT, where it lowers $\text{CHAIR}_i$ to 3.0. On HallusionBench, OPPO also achieves the best fAcc among the compared alignment baselines.

\begin{wraptable}{r}{0.5\textwidth}
\centering
\vspace{-1.25em}
\caption{\small The results on text hallucination benchmarks.} \vspace{-0.5em}
\scalebox{0.75}{
\begin{tabular}{@{}l|*{4}{>{\centering\arraybackslash}p{1.cm}}*{1}{>{\centering\arraybackslash}p{1.35cm}}@{}}
\toprule
 \multirow{2}{*}{Method} & 
        \multicolumn{2}{c}{{KIE-HVQA}} & \multicolumn{2}{c}{TextHalu-Bench} & TextVQA \\  \cmidrule[0.5pt](lr){2-3} \cmidrule[0.5pt](lr){4-5} \cmidrule[0.5pt](l){6-6}
        &Acc.$\uparrow$ & Sim.$\uparrow$ &S-F1$\uparrow$ & U-F1$\uparrow$ &Acc.$\uparrow$ \\
\midrule
Baseline &43.47 & 59.67 & 19.4 & 33.2 & 64.9  \\ 
+ DPO &\underline{28.97} &\underline{42.52}& 20.3 & 35.0 & 65.2  \\
$\triangle$(\textit{vs.} base) & \textcolor{Bittersweet}{-14.50} & \textcolor{Bittersweet}{-17.15} & \textcolor{gray}{+0.9} & \textcolor{gray}{+1.8} & \textcolor{gray}{+0.3} \\
+ mDPO &27.76 & 40.57 & 21.2 & 34.3 & \underline{66.2}  \\
$\triangle$(\textit{vs.} base) & \textcolor{Bittersweet}{-15.71} & \textcolor{Bittersweet}{-19.10} & \textcolor{gray}{+1.8} & \textcolor{gray}{+1.1} & \textcolor{gray}{+1.3} \\
+ CHiP &28.40 & 41.70 & \textbf{21.7} & \textbf{35.6} & 65.9  \\
$\triangle$(\textit{vs.} base) & \textcolor{Bittersweet}{-15.07} & \textcolor{Bittersweet}{-17.97} & \textcolor{gray}{+2.3} & \textcolor{gray}{+2.4} & \textcolor{gray}{+1.0} \\
\rowcolor{gray!10}+ OPPO &\textbf{31.57} & \textbf{47.39} & \underline{21.4} & \underline{34.9} & \textbf{66.3}  \\
$\triangle$(\textit{vs.} base) & \textcolor{Bittersweet}{-11.90} & \textcolor{Bittersweet}{-12.28} & \textcolor{gray}{+2.0} & \textcolor{gray}{+1.7} & \textcolor{gray}{+1.4} \\
\bottomrule
\end{tabular}}
\vspace{-1,em}
\label{tab3:text}
\end{wraptable}
\textbf{Text Scene Hallucination}.
Beyond generalized visual tasks, we investigate the impact of preference optimization on fine-grained text perception using the KIE-HVQA and TextHalu benchmarks. A persistent challenge in multimodal alignment is the \textit{``alignment tax''}, where mitigating hallucinations easily degrades the model's raw recognition capabilities.
Table \ref{tab3:text} shows that OPPO better counters this degradation than competing post-training strategies, retaining stronger accuracy and similarity on KIE-HVQA while remaining competitive on TextHalu. Particularly, text scene benchmarks are more sensitive to local mistakes, and post-training methods can easily improve refusal or caution while damaging recognition quality. OPPO is effective here because its evidence-aware objective encourages the model to capture stronger visual evidence for faithful generation, rather than simply becoming more conservative. 

\begin{wraptable}{l}{0.45\textwidth}
\centering
\vspace{-1.25em}
\caption{\small Results on general-purpose benchmark.} \vspace{-0.4em}
\scalebox{0.75}{
\begin{tabular}{@{}l|*{4}{>{\centering\arraybackslash}p{1.21cm}}@{}}
\toprule
Method & MMBench & MMMU & LLaVAwild & Avg. \\
\midrule
Baseline & 68.1 & 31.8 & 74.7 & 58.2 \\ 
+ DPO & 67.2 & 30.9 & 75.5 & 57.9 \\ 
$\triangle$(\textit{vs.} base) & \textcolor{Bittersweet}{-0.9} & \textcolor{Bittersweet}{-0.9} & \textcolor{gray}{+0.8} & \textcolor{Bittersweet}{-0.3} \\ 
+ mDPO & \underline{67.8} & \underline{31.5} & \underline{75.9} & \underline{58.4} \\ 
$\triangle$(\textit{vs.} base) & \textcolor{Bittersweet}{-0.3} & \textcolor{Bittersweet}{-0.3} & \textcolor{gray}{+1.2} & \textcolor{gray}{+0.2} \\ 
+ CHiP & \textbf{68.3} & 31.0 & 75.8 & \underline{58.4} \\ 
$\triangle$(\textit{vs.} base) & \textcolor{gray}{+0.2} & \textcolor{Bittersweet}{-0.8} & \textcolor{gray}{+1.1} & \textcolor{gray}{+0.2} \\ 
\rowcolor{gray!10}+ OPPO & 67.5 & \textbf{32.3} & \textbf{76.1} & \textbf{58.6} \\ 
$\triangle$(\textit{vs.} base) &\textcolor{Bittersweet}{-0.6} & \textcolor{gray}{+0.5} & \textcolor{gray}{+1.4} &\textcolor{gray}{+0.4} \\ 
\bottomrule
\end{tabular}}
\vspace{-1.15em}
\label{tab4:general}
\end{wraptable}
\textbf{General Capability.} 
To ensure that hallucination-specific tuning does not compromise the model's fundamental reasoning, we report performance on general-purpose benchmarks. As show in Table \ref{tab4:general}, \texttt{OPPO} maintains or even slightly enhances performance on these tasks compared to the baseline. For example, on MMMU, which requires complex domain knowledge, \texttt{OPPO} avoids the performance drop observed in CHiP, suggesting that our objective preserves the model's internal parameterized knowledge.

\textbf{Representation Analysis.} We visualize representation geometry via PCA \cite{mackiewicz1993principal} on AMBER data.  In practice, we extract the average hidden state of the last layer for only image, image with ground-truth, image with wrong answer as inputs, then map them to the same latent space by PCA. As shown in Figure \ref{fig:pca}, the baseline LLaVA-NeXT exhibits highly entangled features, hindering discrimination between facts and fabrications. Among comparators, mDPO separates text types but suffers a modality gap, while CHiP isolates images but fails to disentangle truthful versus hallucinatory texts. In contrast, OPPO produces a clearer separation between image, truthful, and hallucinatory states, which is qualitatively consistent with the intended effect of evidence-aware alignment. 

\textbf{Shifts of Distribution Gaps.}
We analyze how preference optimization shifts the log-probability distributions between accurate and hallucinatory samples. As illustrated in Figure \ref{fig:logpdis}, both DPO and OPPO widen the distribution gap $\Delta$ in vision-conditioned scenarios (I). In the response-contrast setting, OPPO achieves a more pronounced shift of +0.50 compared to +0.44 for DPO, suggesting stronger discrimination when visual evidence is present. We also examine the textual-only conditioned case (II) to estimate residual reliance on language priors. Here OPPO yields a slightly smaller shift than DPO, +0.51 \textit{vs.} +0.53, which is favorable as it suggests that the model is relying less on textual priors alone and more on actual visual evidence when deciding whether a response is trustworthy.

\begin{figure}[t]
  \centering \vspace{-1em}
  \begin{minipage}{0.64\textwidth}
    \centering \hspace{-0.3cm}
    \subfigure[\small$\;\;$Response-Contrast DPO]{
      \includegraphics[width=0.5\linewidth]{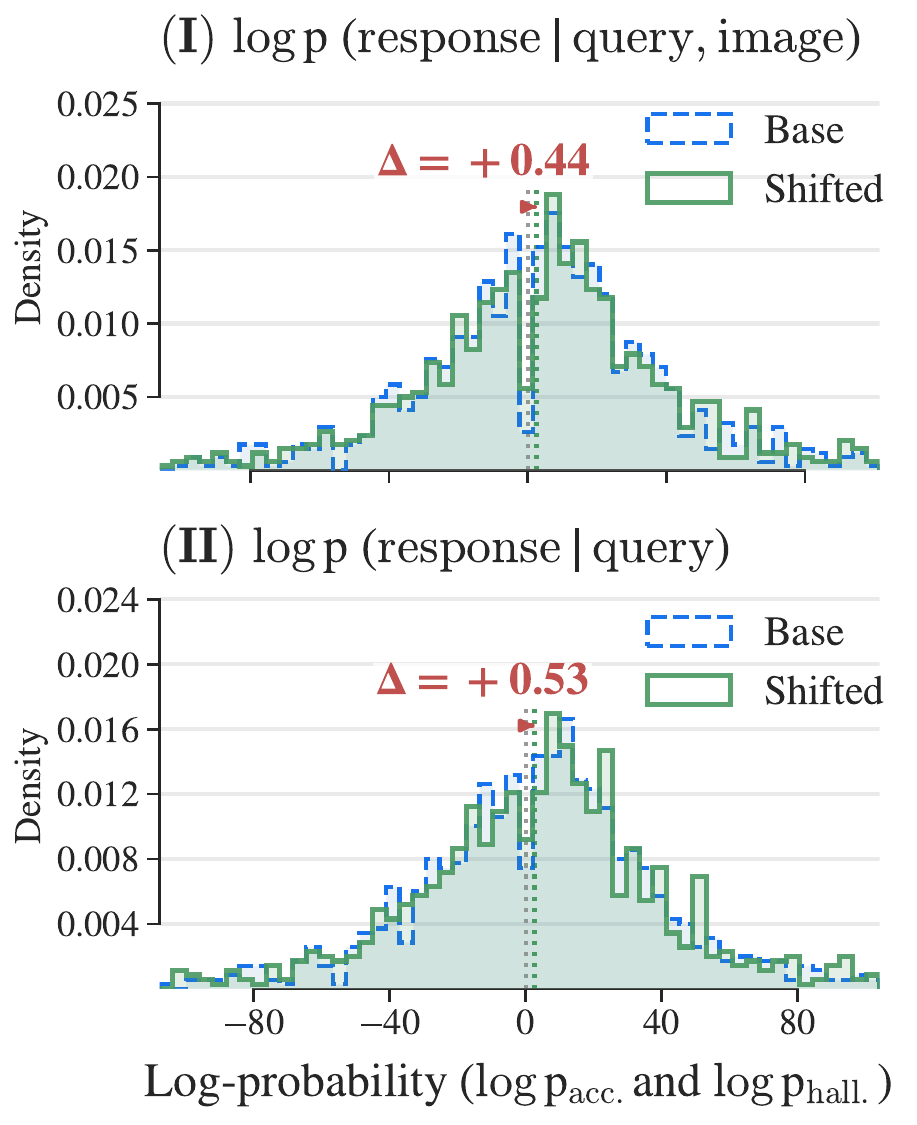}\vspace{-0.1cm} 
    }\hspace{-0.15cm}
    \subfigure[\small$\;\;$Response-Contrast OPPO]{
      \includegraphics[width=0.5\linewidth]{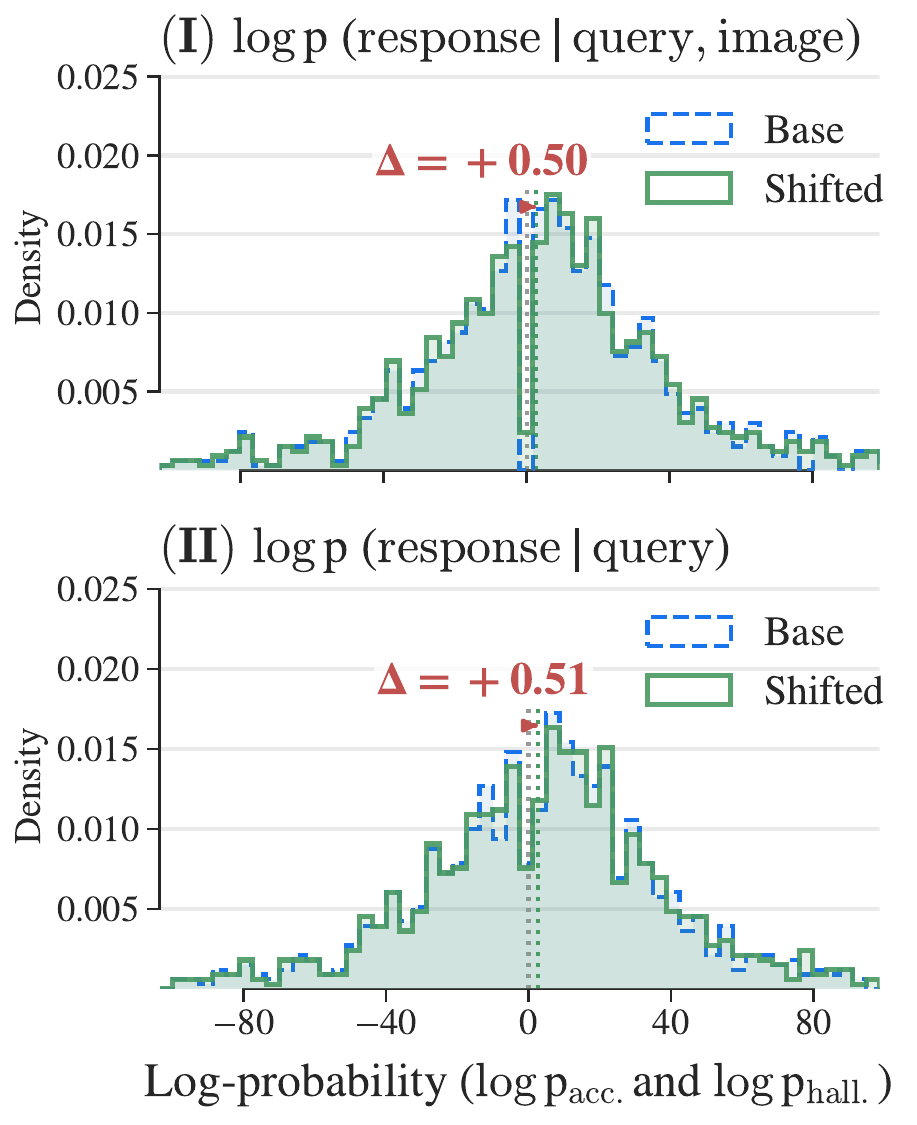}\hspace{-0.15cm} 
    }\vspace{-0.3cm} 
    \caption{Comparison of distribution-gap shifts. The x-axis denotes normalized log-probability, and density denotes the normalized empirical distribution. (I), and (II) illustrate the distributions of vision-conditioned, and textual-only-conditioned generations.}
    \label{fig:logpdis}
  \end{minipage}\hspace{0.7em}
  \begin{minipage}{0.3\textwidth}
    \centering \vspace{0.25cm}
    \subfigure{\hspace{-0.2cm} 
      \includegraphics[width=1\linewidth]{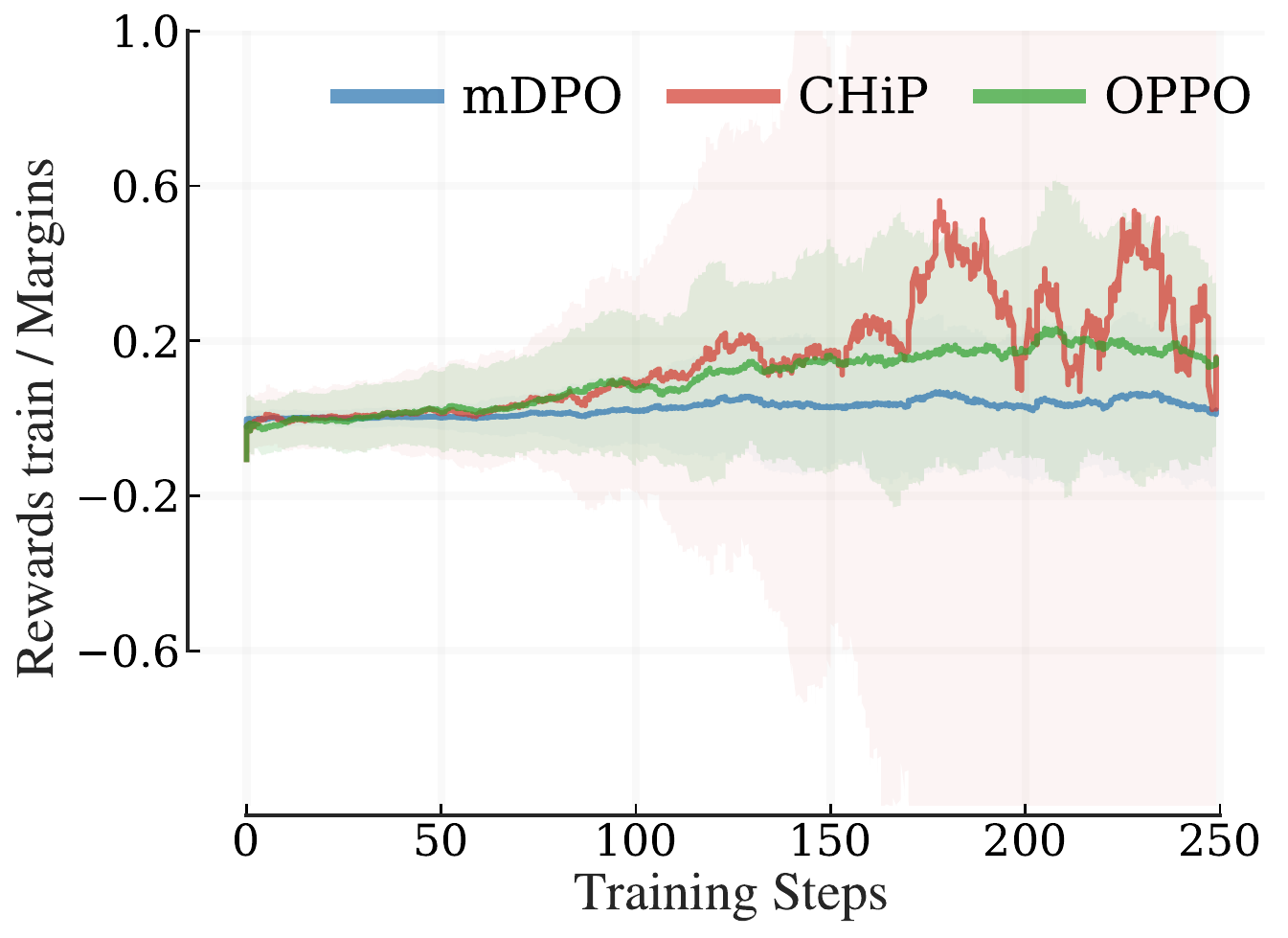} \hspace{-0.15cm} 
    }\\ \vspace{-1.2cm}
    \subfigure{\hspace{-0.2cm} 
      \includegraphics[width=1\linewidth]{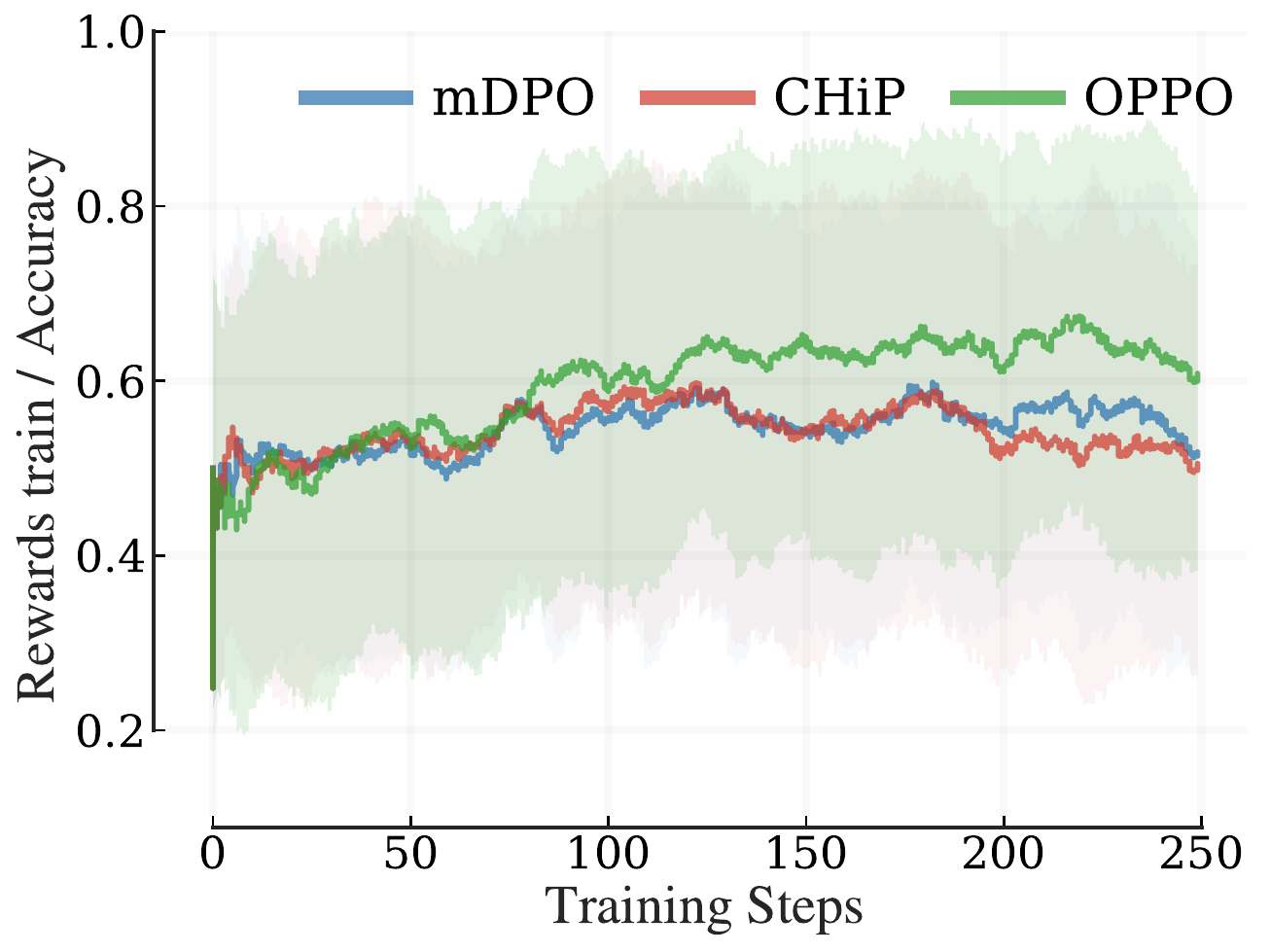}
      \hspace{-0.3cm} 
    }\\ \vspace{-0.05cm}
    {\small(c)$\;\;$Training Dynamics}\vspace{-0.3em}
    \caption{Reward margins and reward accuracies over 250 training steps. OPPO exhibits better stability than baselines.}
    \label{fig:reward_metrics}
  \end{minipage}\vspace{-2em}
\end{figure}

\textbf{Reward Metrics Visualization.}
Beyond static representation analysis, we further investigate training dynamics. We track both metrics over 250 training steps, as illustrated in Figure \ref{fig:reward_metrics}. Here, \emph{reward margin} denotes the implicit-reward difference assigned to the positive and negative elements in a preference pair, and \emph{reward accuracy} denotes the fraction of pairs whose margin has the correct sign. 
The reward margin trajectories reveal distinct learning characteristics. mDPO follows a flat and conservative path, struggling to establish a significant margin between positive and negative pairs. While CHiP achieves higher margins, it exhibits substantial volatility, evidenced by the erratic fluctuations throughout the training process. In contrast, OPPO demonstrates a robust and steady ascent, which without the variance observed in CHiP, indicating a more stable learning signal. Besides, for the reward accuracy in Fig. \ref{fig:reward_metrics} (bottom), we can observe that OPPO consistently outperforms both baselines, achieving higher accuracy levels early in training and maintaining this advantage until convergence, demonstrating superior effectiveness and algorithmic stability during training.

\begin{wraptable}{l}{0.55\textwidth} 
\centering
\vspace{-1.25em}
\caption{\small Ablation study on hallucination-related benchmarks.} \vspace{-0.5em}
\scalebox{0.72}{ 
\begin{tabular}{@{}l|*{6}{>{\centering\arraybackslash}p{1.cm}}@{}}
\toprule
 \multirow{2}{*}{Methods} & 
        POPE & MMHal & AMBER &OBJHal & \multicolumn{2}{c}{HallusionBench}  \\  \cmidrule[0.5pt](lr){2-2} \cmidrule[0.5pt](lr){3-3} \cmidrule[0.5pt](lr){4-4} \cmidrule[0.5pt](lr){5-5} \cmidrule[0.5pt](l){6-7}
        & Acc$\uparrow$ & Score$\uparrow$ & CHAIR$\downarrow$ & CHAIR$_s$$\downarrow$ & $\;\;$qAcc$\uparrow$ & fAcc$\uparrow$ \\
\midrule
LLaVA-NeXT & 85.5 & 3.50 & 8.7 & 11.3 & 4.4 &8.9 \\ 
\midrule
\rowcolor{gray!10} + OPPO & \textbf{87.9} & \textbf{3.54} & \textbf{3.9} & \textbf{4.3} & \textbf{5.6} &\textbf{13.0}\\
$\triangle$ (\textit{vs.} base) &\textcolor{gray}{+2.4}$\;$ & \textcolor{gray}{+0.04}$\;$ & \textcolor{gray}{-4.8}$\;$ & \textcolor{gray}{-7.0}$\;$ & \textcolor{gray}{+1.2}$\;$ & \textcolor{gray}{+4.1}$\;$ \\
w/o $L_{token}$ & 87.0 & 3.51 & 5.1 & 6.6 & 5.4& 12.7 \\ 
w/o $L_{span}$ & 86.8 & 3.50 & 5.4 & 7.0 & 5.1&11.4 \\
w/o $L_{DePO}$ & 87.1 & 3.52 & 4.2 & 5.7 & 5.4 &12.6 \\
w/o $L_{FPO}$ & 86.5 & 3.42 & 4.9 & 9.0 & 5.2 &11.0 \\
\bottomrule
\end{tabular}}
\vspace{-0.5em}
\label{tab:ablation_oppo}
\end{wraptable}
\textbf{Ablation Study.} We conduct a comprehensive ablation study on LLaVA-NeXT to dissect the contribution of each component in \texttt{OPPO}, as summarized in Table \ref{tab:ablation_oppo}. Our findings reveal that the evidence-aware preference is the most critical driver for hallucination reduction. Removing $L_{FPO}$ results in the sharpest performance drop across nearly all metrics, \textit{i.e.}, $\text{CHAIR}_s$ from 4.3 to 9.0, which confirms our hypothesis that explicitly modeling preferences over visual evidence strength is essential for guiding the model away from language priors. The fine-grained regularization framework provide vital stability during the alignment process, $L_{token}$  ensures trajectory-level consistency, the removal of $L_{span}$ significantly impacts performance, which suggests that anchoring preferences to specific answer-bearing spans is crucial for preventing the model from losing track of visual details during long-form generation. 

\begin{wraptable}{r}{0.47\textwidth}
\centering \footnotesize
  \renewcommand\tabcolsep{1pt}
  \vspace{-1.5em}
  \caption{Results on different hyperparameters.}
\label{tab6:hyperparameter}
  \vspace{-0.6em}
    \resizebox{0.47\textwidth}{!}
    {
   \begin{tabular}{l*{1}{>{\centering\arraybackslash}p{3.2em}}*{1}{>{\centering\arraybackslash}p{3.6em}}*{1}{>{\centering\arraybackslash}p{4em}}*{1}{>{\centering\arraybackslash}p{3.em}}}
    \toprule
    \multirow{2}{*}{Method} & MMHal & AMBER & OBJHal\;\; & Avg. $\downarrow$ \\ \cmidrule[0.5pt](lr){2-2} \cmidrule[0.5pt](lr){3-3} \cmidrule[0.5pt](lr){4-4}
    &Hal. $\downarrow$ &Hal. $\downarrow$ & CHAIR$_s$$\downarrow$ \\
    \midrule
    $\beta_1 = 0.01, \beta_2 = 0.1$ & 58.4 & 22.1 & 33.2\; & 37.90 \\
    $\beta_1 = 0.05, \beta_2 = 0.1$ & 54.8 & 18.5 & 30.6\; & 34.63 \\
    \rowcolor{gray!10} $\beta_1 = 0.1, \beta_2 = 0.1$ & 52.0 & 14.9 & 28.9\; & \textbf{31.93} \\
    $\beta_2 = 0.01, \beta_1 = 0.1$ & 57.3 & 25.6 & 38.4\; & 40.43 \\
    $\beta_2 = 0.05, \beta_1 = 0.1$ & 53.9 & 17.2 & 31.5\; & 34.20 \\
    \bottomrule
    \end{tabular}
    }
  \vspace{-1.25em}
\end{wraptable}
\textbf{Hyperparameter Analysis.} We investigate the sensitivity of \texttt{OPPO} to its two primary hyperparameters: $\beta_1$, which controls the strength of the evidence-aware preference, and $\beta_2$, which governs the fine-grained textual regularization. As shown in Table \ref{tab6:hyperparameter}, we observe that increasing $\beta_1$ from 0.01 to 0.1 significantly reduces the hallucination rate across all benchmarks, with the average score dropping from 37.90 to 31.93. This trend underscores the importance of the evidence-level objective in forcing the model to prioritize visual cues. Besides, a very low $\beta_2$ leads to a noticeable performance degradation, suggesting that fine-grained textual grounding is necessary to prevent the model from drifting during preference optimization.

\begin{wraptable}{r}{0.4\textwidth}
\centering \footnotesize
  \renewcommand\tabcolsep{2pt}
  \vspace{-1.5em}
  \caption{Robustness to attention priors.}
\label{tab:attention_robustness}
  \vspace{-0.6em}
    \resizebox{0.4\textwidth}{!}
    {
   \begin{tabular}{l*{3}{>{\centering\arraybackslash}p{3.6em}}}
    \toprule
    Method& POPE$\uparrow$ & MMHal$\uparrow$ & OBJHal$\downarrow$ \\ 
    \midrule
    LLaVA-NeXT & 85.5 & 3.50 & 11.3 \\
    + mDPO  & 87.8 & 3.49 & 8.7 \\
    \rowcolor{gray!10}  + OPPO$\dagger$ & 86.4 & 3.48 & 7.5 \\
    \rowcolor{gray!10} + OPPO (naive) & \textbf{87.9} & \textbf{3.54} & \textbf{4.3} \\
    \bottomrule
    \end{tabular}
    }
  \vspace{-1.25em}
\end{wraptable}
\textbf{Robustness to Attention Priors.} 
To further verify whether magnifying flawed attention priors in $v_w$ causes error amplification, we address this empirically and structurally. First, modern MLLMs possess strong baseline grounding; our analysis shows only $\sim$5.4\% of attention peaks in our training set severely misalign with ground-truth bounding boxes. Second, we test an adversarial variant, OPPO$\dagger$, which applies random spatial shifts to the heatmaps to forcibly magnify irrelevant regions. As shown in Table \ref{tab:attention_robustness}, while performance drops compared to standard OPPO, it avoids catastrophic collapse, which demonstrates the effectiveness of our strategy.

As shown in Fig.~\ref{fig:cases}, \texttt{OPPO} produces more visually grounded predictions across multiple failure modes, including scene-text recognition, object presence verification, counting, and local attribute reasoning.

\vspace{-0.5em}\section{Related Work}\vspace{-0.25em}
Recent efforts to mitigate MLLM hallucination include specialized fine-tuning \cite{gunjal2024detecting}, preference optimization like RLHF \cite{yu2024rlhf} and DPO \cite{rafailov2023direct}, and inference-time interventions such as OPERA \cite{opera2024cvpr} and MemVR\cite{zoulook}. In text-rich scenarios, models often suffer from non-semantic hallucinations \cite{liu2023hidden} despite improvements in resolution \cite{li2024monkey}. While recent works address text-specific issues through layer activation correction \cite{shu2025semantics} or refusal strategies \cite{he2025seeing}, they often introduce latency or domain-specific constraints. Unlike existing coarse-grained preference optimization \cite{sun2024aligning, shao2024deepseekmath,xie2024v,wang2024mdpo,liu2025mitigating}, our work proposes an evidence-aware preference optimization framework that also transfers to demanding visual grounding settings. Please refer to the Appendix \ref{sec:relatedwork} for a comprehensive literature review.
\begin{figure}
    \centering \vspace{-0.75em}
    \includegraphics[width=1\linewidth]{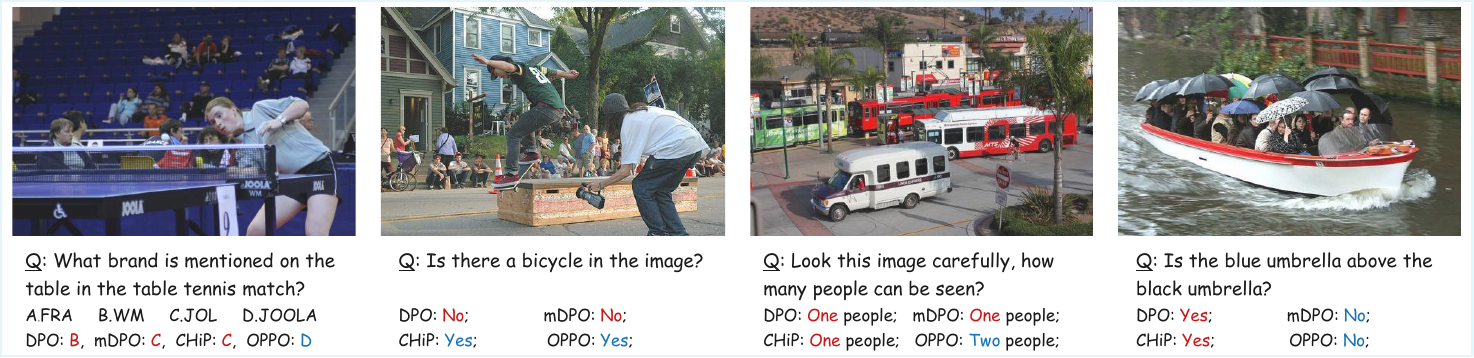}\vspace{-0.25em}
    \caption{Qualitative cases showing improved visual grounding with \texttt{OPPO} compared with baselines.}
    \label{fig:cases} \vspace{-1em}
\end{figure}

\vspace{-0.5em}
\section{Limitation}
\vspace{-0.25em}
For the performance drop on KIE-HVQA, we believe it is partly due to the nature of this benchmark: its prompts are substantially longer and more instruction-heavy than those in standard OCR-style evaluation, and the benchmark is primarily designed to assess very large models with stronger long-context instruction-following ability. Under this setting, post-training a 7B-scale model for hallucination mitigation may introduce a mild instruction-following tax, which can in turn reduce performance on tasks that rely heavily on parsing long and complex prompts, even when visual grounding itself is improved. Our empirical study is still limited to a small number of model families and a fixed set of hyperparameters, so broader cross-family validation, stronger robustness analysis, and more extensive sensitivity studies remain important directions for future work.

\vspace{-0.5em}
\section{Conclusion}\vspace{-0.25em}
We revisit multimodal hallucination through the lens of evidence utilization. Motivated by the empirical finding that strengthening query-relevant attended evidence improves grounded generation, we propose \texttt{OPPO}, an evidence-aware preference optimization framework that aligns response preference with ordered visual evidence strength. Across hallucination and general-purpose benchmarks, \texttt{OPPO} consistently outperforms strong baselines under matched settings, while our theoretical analysis suggests a positive lower bound on local visual sensitivity along the constructed evidence path, which demonstrates evidence-aware alignment as a promising direction for mitigating hallucination.

\clearpage
{\small
\bibliographystyle{plain}
\bibliography{example_paper}
}

\newpage

\appendix
\onecolumn

\section{Related Work}\label{sec:relatedwork}
\textbf{Mitigating Generalized Hallucination in MLLMs.}
Extensive research has investigated the origins of hallucinations in MLLMs \cite{yin2023survey, zhou2023analyzing, bai2024hallucination}. Early efforts focused on fine-grained modality alignment \cite{rohrbach2018object} and mitigating co-occurrence biases \cite{kim2023exposing} in small-scale models. Recent strategies include specialized fine-tuning \cite{gunjal2024detecting}, RLHF \cite{yu2024rlhf}, DPO \cite{rafailov2023direct}, and post-hoc revisors like LURE \cite{zhou2023analyzing}, which edits potential hallucinations. Besides, there are attentional intervention methods like OPERA \cite{opera2024cvpr} that mitigate hallucinations without extra data or training \cite{zhang2024seeing, xing2024mitigating}, yet they introduce higher inference latency. Alternatively, Contrastive Decoding (CD) methods, such as VCD \cite{vcd2024cvpr} and ICD \cite{wang2024mitigating}, adjust the logit distribution to suppress hallucinations \cite{chuang2023dola, chenhalc, neo2024vord}. Nevertheless, CD-based approaches can introduce potential noise into the distribution, leading to inconsistent performance improvements. More recently, MemVR \cite{zoulook} employs memory-space visual retracing for output verification, but such iterative refinement also comes at the cost of increased inference time. To address the computational overhead of inference-time interventions, recent advancements have pivoted towards efficient training-time alignment, particularly leveraging DPO to fundamentally reduce hallucinatory tendencies. Early attempts like HA-DPO \cite{zhao2023beyond} formulate hallucination mitigation as a preference ranking problem, constructing positive and negative pairs to directly penalize hallucinatory content. Building on this, V-DPO \cite{xie2024v} argues that over-reliance on language priors is a primary cause of hallucination and introduces vision-guided preference learning to better anchor generations to the image context. Moving towards more granular control, HDPO \cite{fu2025mitigating} categorizes and targets specific failure modes, such as multimodal conflicts and long-context degradation, by constructing targeted preference data, while \cite{xiao2025detecting} utilize fine-grained AI feedback to synthesize high-quality preference datasets via a detect-then-rewrite pipeline. From a theoretical perspective, recent studies identify optimization bottlenecks in standard DPO when applied to MLLMs. OPA-DPO \cite{yang2025mitigating} reveals that the distribution shift from off-policy data hinders effective alignment, and proposes an on-policy strategy to bridge this gap. Similarly, SymMPO \cite{liu2025mitigating} introduces a symmetric preference optimization objective with margin consistency to ensure rigorous theoretical alignment, thereby achieving more robust hallucination mitigation. P$^2$DPO \cite{p2DPO} integrates perceptual processing with calibration-based preference optimization to further ground model generations in vision and mitigate hallucination.

\textbf{OCR Hallucination and Its Mitigation.} 
Text-rich scenarios present unique grounding challenges, where models often succumb to “semantic hallucination”, generating plausible but visually incorrect text driven by language priors rather than visual evidence \cite{liu2023hidden}. Comprehensive evaluations in benchmarks like OCRBench \cite{liu2023ocrbench} -v2\cite{fu2024ocrbenchv2improvedbenchmark} and TextSquare \cite{tang2024textsquare} reveal that standard MLLMs struggle significantly with dense text. While scaling input resolution \cite{li2024monkey, ye2023mplug} improves recognition, specific hallucination issues persist. Very recently, \cite{shu2025semantics} addressed this by identifying conflicts between semantic and visual cues, proposing a training-free intervention to correct layer activations, but such inference-time mechanisms inevitably increase latency. Besides, \cite{he2025seeing} tackled OCR hallucinations in degraded documents by incorporating uncertainty-aware refusal strategies. Nevertheless, their method focuses on abstention rather than correction and is primarily tailored to specific domains like invoices. Beyond existing studies, we establish a comprehensive benchmark for evaluation, and propose the first oriented pickup preference optimization algorithm to mitigate OCR hallucination.

\textbf{Preference Optimization.} The alignment of MLLMs with human intent has evolved significantly, transitioning from standard Supervised Fine-Tuning (SFT) to more robust Preference Optimization paradigms. Traditional approaches predominantly employ Reinforcement Learning from Human Feedback (RLHF) \cite{ouyang2022training}, which typically leverages Proximal Policy Optimization (PPO) \cite{schulman2017proximal} to maximize expected rewards under KL-divergence constraints. However, PPO is often plagued by training instability, hyperparameter sensitivity, and the high computational overhead of maintaining multiple models (e.g., actor, critic, and reward models). To mitigate these inefficiencies, Group Relative Policy Optimization (GRPO) \cite{shao2024deepseekmath} has emerged as a compelling alternative, estimating baselines from group scores to eliminate the dependency on an explicit value function. Parallelly, Direct Preference Optimization (DPO) \cite{rafailov2023direct} has revolutionized the field by deriving a closed-form solution that reparameterizes the reward function, thereby transforming the complex RL problem into a stable, offline binary classification objective. In the multimodal landscape, recent works such as RLHF-V \cite{yu2024rlhf} and LLaVA-RLHF \cite{sun2024aligning} have successfully adapted these techniques to suppress hallucinations by penalizing non-factual responses. Nevertheless, these existing methods generally rely on coarse-grained rejection at the response or segment level. They treat the visual input as a static condition rather than an active optimization target, lacking the targeted mechanism to enforce fine-grained visual grounding \cite{zou2023dpnet,zou2026don,zou2026learning}.

\paragraph{Feature-level comparison.}
To make the novelty boundary explicit, \autoref{tab:method_feature_comparison} contrasts OPPO with the most relevant diagnostic, inference-time, and preference-optimization approaches. The key difference is structural: OPPO does not merely identify attended evidence or build visual preference pairs; it orders multiple semantically consistent views by evidence strength and aligns the policy to scale its preference for the same faithful response along this evidence axis, further stabilized by span- and token-level calibration.

\begin{table}[t]
\centering
\small
\setlength{\tabcolsep}{3.1pt}
\renewcommand{\arraystretch}{1.13}
\providecommand{\cmark}{\ding{51}}
\providecommand{\xmark}{\ding{55}}
\providecommand{\pmark}{\raisebox{0.2ex}{\scalebox{0.8}{$\triangle$}}}
\caption{Feature-level comparison with closely related hallucination-mitigation and multimodal preference-optimization methods. $\triangle$ denotes partial or implicit coverage. OPPO is distinguished by converting query-aware visual evidence into an ordered preference axis and coupling it with fine-grained textual calibration.}
\label{tab:method_feature_comparison}
\resizebox{\linewidth}{!}{%
\begin{tabular}{lcccccc>{\centering\arraybackslash}p{3.3cm}}
\toprule
\textbf{Method} &
\makecell{\textbf{Policy-level}\\\textbf{evidence gap}} &
\makecell{\textbf{Query-aware}\\\textbf{evidence}} &
\makecell{\textbf{Stronger/}\\\textbf{weaker views}} &
\makecell{\textbf{Ordered}\\\textbf{evidence pref.}} &
\makecell{\textbf{Visual}\\\textbf{pref. opt.}} &
\makecell{\textbf{Span/token}\\\textbf{calibration}} &
\makecell{\textbf{Main mechanism}} \\
\midrule
DPO \cite{rafailov2023direct} & \xmark & \xmark & \xmark & \xmark & \xmark & \xmark & Response-level preference under fixed inputs \\
mDPO \cite{wang2024mdpo} & \pmark & \pmark & \xmark & \xmark & \cmark & \xmark & Conditional multimodal preference over response pairs \\
V-DPO \cite{xie2024v} & \pmark & \pmark & \pmark & \xmark & \cmark & \xmark & Vision-guided DPO to reduce language-prior hallucination \\
CHiP \cite{fuchip} & \pmark & \pmark & \xmark & \xmark & \cmark & \pmark & Cross-modal hierarchical preference optimization \\
SymMPO \cite{liu2025mitigating} & \pmark & \xmark & \xmark & \xmark & \cmark & \xmark & Symmetric preference margins for hallucination mitigation \\
OPERA \cite{opera2024cvpr} & \pmark & \cmark & \xmark & \xmark & \xmark & \xmark & Inference-time attention penalty and retrospection allocation \\
MLLMs Know Where to Look \cite{zhang2025mllms} & \pmark & \cmark & \pmark & \xmark & \xmark & \xmark & Training-free attention-guided perception of small details \\
Seeing but Not Believing \cite{liu2025seeingnotbelieving} & \cmark & \cmark & \pmark & \xmark & \xmark & \xmark & Diagnostic analysis and inference-time evidence highlighting \\
\rowcolor{gray!10}\textbf{OPPO (ours)} & \cmark & \cmark & \cmark & \cmark & \cmark & \cmark & Ordered evidence pickup with span- and token-level calibration \\
\bottomrule
\end{tabular}%
}
\vspace{-0.25em}
\end{table}

\newpage

\section{Theoretical Analysis of OPPO}

We analyze the evidence-aware component of OPPO, since this is the term that directly encodes ordered visual-evidence preferences. For a fixed prompt $x$ and a faithful response $y_w$, define the log-ratio score
\begin{equation}
\ell_\theta(v)
\;:=\;
\log \frac{\pi_\theta(y_w \mid x, v)}{\pi_{\mathrm{ref}}(y_w \mid x, v)}.
\label{eq:score_def}
\end{equation}
For any ordered pair of views $u \succ v$, define the pairwise evidence margin
\begin{equation}
\Delta_{u,v}(\theta)
\;:=\;
\gamma_{u,v}\bigl(\ell_\theta(u)-\ell_\theta(v)\bigr),
\qquad
\gamma_{u,v}>0,
\label{eq:pair_margin}
\end{equation}
and the corresponding pairwise evidence-ranking loss
\begin{equation}
\mathcal{L}_{u \succ v}(\theta)
\;:=\;
-\log \sigma\!\bigl(\Delta_{u,v}(\theta)\bigr).
\label{eq:pair_loss}
\end{equation}
In OPPO, the evidence-aware objective is
\begin{equation}
\mathcal{L}_{\mathrm{EV}}(\theta)
=
\mathcal{L}_{v_w \succ v_a}(\theta)
+
\mathcal{L}_{v_a \succ v_l}(\theta),
\label{eq:ev_loss}
\end{equation}
where $(v_w, v_a, v_l)$ denote the stronger-evidence, anchored, and weaker-evidence views, respectively.

\paragraph{Remark.}
Eq.~\eqref{eq:score_def} is exactly the quantity used by the evidence-aware preference term in the main method. Therefore, the analysis below directly characterizes the optimization behavior of OPPO's ordered evidence objective.

\begin{assumption}[Semantic invariance of the view chain]
\label{ass:semantic_invariance}
The three constructed views $(v_w, v_a, v_l)$ preserve the same semantic target for the query $x$, i.e., the faithful response $y_w$ remains valid across all three views, while only the effective strength of supporting visual evidence changes.
\end{assumption}

\begin{assumption}[Weak-view prior proxy]
\label{ass:prior_proxy}
The weaker-evidence view $v_l$ substantially attenuates the query-relevant evidence, so that $\ell_\theta(v_l)$ is more dominated by language prior than by grounded visual support. This assumption is approximate and is used only to interpret the optimization direction on the weak view.
\end{assumption}

\begin{assumption}[Smoothness]
\label{ass:smoothness}
For fixed $(x,y_w)$, the map $v \mapsto \ell_\theta(v)$ is continuously differentiable on every line segment joining two views in $\{v_w, v_a, v_l\}$.
\end{assumption}

\begin{proposition}[Score-level monotonicity of the pairwise evidence loss]
\label{prop:score_monotonicity}
For any ordered pair $u \succ v$, the pairwise loss $\mathcal{L}_{u \succ v}$ satisfies
\begin{equation}
\frac{\partial \mathcal{L}_{u \succ v}}{\partial \ell_\theta(u)}
=
-\gamma_{u,v}\bigl(1-\sigma(\Delta_{u,v})\bigr)
< 0,
\qquad
\frac{\partial \mathcal{L}_{u \succ v}}{\partial \ell_\theta(v)}
=
\gamma_{u,v}\bigl(1-\sigma(\Delta_{u,v})\bigr)
> 0.
\label{eq:score_derivatives}
\end{equation}
Hence, gradient descent on $\mathcal{L}_{u \succ v}$ pushes the preferred-view score $\ell_\theta(u)$ upward and the dispreferred-view score $\ell_\theta(v)$ downward.
\end{proposition}

\begin{proof}
Let $\Delta = \Delta_{u,v}(\theta) = \gamma_{u,v}(\ell_\theta(u)-\ell_\theta(v))$. Since
\[
\mathcal{L}_{u \succ v} = -\log \sigma(\Delta),
\]
we have
\[
\frac{d}{d\Delta}\bigl(-\log \sigma(\Delta)\bigr)
=
-(1-\sigma(\Delta)).
\]
Applying the chain rule gives
\[
\frac{\partial \mathcal{L}_{u \succ v}}{\partial \ell_\theta(u)}
=
-(1-\sigma(\Delta)) \cdot \gamma_{u,v},
\]
and
\[
\frac{\partial \mathcal{L}_{u \succ v}}{\partial \ell_\theta(v)}
=
-(1-\sigma(\Delta)) \cdot (-\gamma_{u,v})
=
\gamma_{u,v}(1-\sigma(\Delta)).
\]
Since $\gamma_{u,v}>0$ and $0<\sigma(\Delta)<1$, the first derivative is strictly negative and the second is strictly positive. Therefore, gradient descent increases $\ell_\theta(u)$ and decreases $\ell_\theta(v)$.
\end{proof}

\begin{lemma}[First-order margin amplification under parameter updates]
\label{lem:margin_amplification}
Assume $\Delta_{u,v}(\theta)$ is twice continuously differentiable with respect to $\theta$. Let
\begin{equation}
\theta^{+}
=
\theta
-
\eta \nabla_\theta \mathcal{L}_{u \succ v}(\theta),
\qquad
\eta>0.
\label{eq:gd_update}
\end{equation}
Then the pairwise margin satisfies the first-order expansion
\begin{equation}
\Delta_{u,v}(\theta^{+})
=
\Delta_{u,v}(\theta)
+
\eta \bigl(1-\sigma(\Delta_{u,v}(\theta))\bigr)
\left\|
\nabla_\theta \Delta_{u,v}(\theta)
\right\|^2
+
O(\eta^2).
\label{eq:margin_first_order}
\end{equation}
In particular, for sufficiently small $\eta$, if $\nabla_\theta \Delta_{u,v}(\theta)\neq 0$, then one gradient step strictly increases the margin $\Delta_{u,v}$.
\end{lemma}

\begin{proof}
From Eq.~\eqref{eq:pair_loss},
\[
\nabla_\theta \mathcal{L}_{u \succ v}(\theta)
=
-(1-\sigma(\Delta_{u,v}(\theta)))
\nabla_\theta \Delta_{u,v}(\theta).
\]
Therefore,
\[
\theta^{+}-\theta
=
\eta \bigl(1-\sigma(\Delta_{u,v}(\theta))\bigr)
\nabla_\theta \Delta_{u,v}(\theta).
\]
Applying the first-order Taylor expansion of $\Delta_{u,v}$ around $\theta$ yields
\[
\Delta_{u,v}(\theta^{+})
=
\Delta_{u,v}(\theta)
+
\nabla_\theta \Delta_{u,v}(\theta)^\top (\theta^{+}-\theta)
+
O(\|\theta^{+}-\theta\|^2).
\]
Substituting the update expression above gives
\[
\Delta_{u,v}(\theta^{+})
=
\Delta_{u,v}(\theta)
+
\eta \bigl(1-\sigma(\Delta_{u,v}(\theta))\bigr)
\left\|
\nabla_\theta \Delta_{u,v}(\theta)
\right\|^2
+
O(\eta^2),
\]
which proves the claim.
\end{proof}

\begin{corollary}[Relative suppression of prior-supported confidence on the weak view]
\label{cor:prior_suppression}
Consider the anchored-vs-weak pair $v_a \succ v_l$. Minimizing $\mathcal{L}_{v_a \succ v_l}$ strictly increases the gap
\begin{equation}
\ell_\theta(v_a)-\ell_\theta(v_l).
\label{eq:weak_gap}
\end{equation}
Under Assumption~\ref{ass:prior_proxy}, this comparatively suppresses confidence that can be maintained under weak, prior-dominated evidence. Moreover, if $\ell_\theta(v_a)$ is locally unchanged to first order, then $\ell_\theta(v_l)$ decreases strictly after a gradient step.
\end{corollary}

\begin{proof}
The first statement follows directly from Proposition~\ref{prop:score_monotonicity} and Lemma~\ref{lem:margin_amplification}. Under Assumption~\ref{ass:prior_proxy}, the score on $v_l$ is interpreted as being increasingly dominated by language prior once grounded evidence is weakened. Therefore, enlarging the anchored-vs-weak gap suppresses prior-supported confidence \emph{relative} to the anchored grounded score. If $\ell_\theta(v_a)$ is locally fixed, then the increase of the gap in Eq.~\eqref{eq:weak_gap} must come from a strict decrease in $\ell_\theta(v_l)$.
\end{proof}

\begin{lemma}[Positive directional sensitivity along the enhancement direction]
\label{lem:directional_sensitivity}
Let
\begin{equation}
\delta_{+} := v_w - v_a.
\end{equation}
Under Assumption~\ref{ass:smoothness},
\begin{equation}
\ell_\theta(v_w)-\ell_\theta(v_a)
=
\int_0^1
\nabla_v \ell_\theta(v_a+t\delta_{+})^\top \delta_{+}\, dt.
\label{eq:ftc_plus}
\end{equation}
Consequently, if
\begin{equation}
\ell_\theta(v_w)-\ell_\theta(v_a) \ge \varepsilon_{+} > 0,
\label{eq:plus_gap}
\end{equation}
then there exists a point $\xi_{+}$ on the line segment $[v_a,v_w]$ such that
\begin{equation}
\nabla_v \ell_\theta(\xi_{+})^\top \delta_{+}
=
\ell_\theta(v_w)-\ell_\theta(v_a)
\ge
\varepsilon_{+},
\label{eq:directional_gap}
\end{equation}
and therefore
\begin{equation}
\left\|
\nabla_v \ell_\theta(\xi_{+})
\right\|
\ge
\frac{\varepsilon_{+}}{\|\delta_{+}\|}.
\label{eq:directional_lower_bound}
\end{equation}
\end{lemma}

\begin{proof}
Define the scalar function
\[
g_{+}(t) := \ell_\theta(v_a+t\delta_{+}), \qquad t\in[0,1].
\]
By Assumption~\ref{ass:smoothness}, $g_{+}$ is continuously differentiable. Applying the fundamental theorem of calculus,
\[
\ell_\theta(v_w)-\ell_\theta(v_a)
=
g_{+}(1)-g_{+}(0)
=
\int_0^1 g_{+}'(t)\,dt
=
\int_0^1
\nabla_v \ell_\theta(v_a+t\delta_{+})^\top \delta_{+}\,dt,
\]
which proves Eq.~\eqref{eq:ftc_plus}. Since the integrand is continuous, the mean value theorem for integrals implies that there exists $t^\star\in(0,1)$ such that
\[
\nabla_v \ell_\theta(v_a+t^\star\delta_{+})^\top \delta_{+}
=
\ell_\theta(v_w)-\ell_\theta(v_a).
\]
Let $\xi_{+}=v_a+t^\star\delta_{+}$. Then Eq.~\eqref{eq:directional_gap} follows from Eq.~\eqref{eq:plus_gap}. Finally, by Cauchy--Schwarz,
\[
\nabla_v \ell_\theta(\xi_{+})^\top \delta_{+}
\le
\left\|\nabla_v \ell_\theta(\xi_{+})\right\|\,\|\delta_{+}\|,
\]
which implies Eq.~\eqref{eq:directional_lower_bound}.
\end{proof}

\begin{proposition}[Transitivity of attained evidence margins]
\label{prop:transitivity}
Suppose the two OPPO evidence margins satisfy
\begin{equation}
\Delta_{v_w,v_a}(\theta) \ge m_1 > 0,
\qquad
\Delta_{v_a,v_l}(\theta) \ge m_2 > 0.
\label{eq:two_margins}
\end{equation}
Then
\begin{equation}
\ell_\theta(v_w) > \ell_\theta(v_a) > \ell_\theta(v_l),
\label{eq:ordered_scores}
\end{equation}
and
\begin{equation}
\ell_\theta(v_w)-\ell_\theta(v_l)
\ge
\frac{m_1}{\gamma_{v_w,v_a}}
+
\frac{m_2}{\gamma_{v_a,v_l}}.
\label{eq:total_gap}
\end{equation}
\end{proposition}

\begin{proof}
By the definition of the margins in Eq.~\eqref{eq:pair_margin},
\[
\ell_\theta(v_w)-\ell_\theta(v_a)
=
\frac{\Delta_{v_w,v_a}(\theta)}{\gamma_{v_w,v_a}}
\ge
\frac{m_1}{\gamma_{v_w,v_a}}
>0,
\]
and similarly
\[
\ell_\theta(v_a)-\ell_\theta(v_l)
=
\frac{\Delta_{v_a,v_l}(\theta)}{\gamma_{v_a,v_l}}
\ge
\frac{m_2}{\gamma_{v_a,v_l}}
>0.
\]
This proves Eq.~\eqref{eq:ordered_scores}. Adding the two inequalities gives Eq.~\eqref{eq:total_gap}.
\end{proof}

\begin{theorem}[Local non-vanishing visual sensitivity along the evidence path]
\label{thm:local_sensitivity}
Under Assumption~\ref{ass:smoothness}, suppose Eq.~\eqref{eq:two_margins} holds. Let
\begin{equation}
\delta := v_w - v_l.
\end{equation}
Then there exists a point $\xi$ on the line segment $[v_l,v_w]$ such that
\begin{equation}
\left\|
\nabla_v \ell_\theta(\xi)
\right\|
\ge
\frac{1}{\|v_w-v_l\|}
\left(
\frac{m_1}{\gamma_{v_w,v_a}}
+
\frac{m_2}{\gamma_{v_a,v_l}}
\right)
> 0.
\label{eq:local_lower_bound}
\end{equation}
Equivalently, any solution that achieves positive ordered evidence margins cannot be visually insensitive everywhere along the constructed evidence path from $v_l$ to $v_w$.
\end{theorem}

\begin{proof}
Define
\[
g(t) := \ell_\theta(v_l + t\delta), \qquad t\in[0,1].
\]
By Assumption~\ref{ass:smoothness}, $g$ is continuously differentiable. By the one-dimensional mean value theorem, there exists $t^\star\in(0,1)$ such that
\begin{equation}
g'(t^\star)
=
g(1)-g(0)
=
\ell_\theta(v_w)-\ell_\theta(v_l).
\label{eq:mvt_main}
\end{equation}
Let $\xi = v_l + t^\star \delta$. Then
\begin{equation}
\nabla_v \ell_\theta(\xi)^\top (v_w-v_l)
=
\ell_\theta(v_w)-\ell_\theta(v_l).
\label{eq:path_derivative}
\end{equation}
By Proposition~\ref{prop:transitivity},
\[
\ell_\theta(v_w)-\ell_\theta(v_l)
\ge
\frac{m_1}{\gamma_{v_w,v_a}}
+
\frac{m_2}{\gamma_{v_a,v_l}}.
\]
Applying Cauchy--Schwarz to Eq.~\eqref{eq:path_derivative} yields
\[
\left\|
\nabla_v \ell_\theta(\xi)
\right\|
\,\|v_w-v_l\|
\ge
\ell_\theta(v_w)-\ell_\theta(v_l)
\ge
\frac{m_1}{\gamma_{v_w,v_a}}
+
\frac{m_2}{\gamma_{v_a,v_l}}.
\]
Rearranging proves Eq.~\eqref{eq:local_lower_bound}.
\end{proof}

\paragraph{Interpretation.}
Theorem~\ref{thm:local_sensitivity} is intentionally local: it does \emph{not} claim that OPPO globally eliminates hallucination. Instead, it shows that once the ordered evidence margins are achieved, the optimized model must exhibit non-zero visual sensitivity at least somewhere along the constructed path from weak to strong evidence. This is exactly the behavior OPPO is designed to encourage.

\section{Benchmarks and Evaluation Metrics.}\label{apx:benchmarks}

\textbf{Object HalBench (ObjHal)}~\cite{rohrbach2018object} is a widely used benchmark for evaluating object hallucination.  
To improve evaluation stability, the benchmark includes 8 diverse prompts and is tested on 300 instances. 
\textbf{Metrics}: Following~\cite{yu2024rlhf,wang2024mdpo}, we report both the \textit{response-level hallucination rate} (\texttt{R.}) and \textit{mention-level hallucination rate} (\texttt{M.}).

\textbf{MMHal-Bench}~\cite{sun2023aligning} is a question-answering benchmark that covers 8 question categories and 12 object topics. 
    \textbf{Metrics:} It uses GPT-4 to assess response quality (\texttt{Ova.}) and hallucination rates (\texttt{R.}).
    
\textbf{HallusionBench}~\cite{guan2024hallusionbench} 
evaluates visual illusions and knowledge hallucinations, featuring 346 images and 1129 questions.
It was the GPT4-assisted evaluation.
\textbf{Metrics:} 
Question Pair Accuracy (\texttt{qA}), 
Figure Accuracy (\texttt{fA}), 
and
All Accuracy (\texttt{aA}). 

\textbf{AMBER}~\cite{wang2023llm} was designed to be evaluated without LLM assistance. 
Following previous works~\cite{wang2024mdpo}, we only consider the generative tasks.
\textbf{Metrics:} 
(a) CHAIR~\cite{rohrbach2018object} (\texttt{CHAIR});
(b) Object coverage of responses (\texttt{Cover});
(c) Response-level hallucination (\texttt{Hal});
(d) Human cognition hallucination (\texttt{Cog}).

\textbf{Polling based Object Probing Evaluation (POPE)}~\cite{li2023evaluating} is a VQA-based metric proposed to assess hallucinations in MLLMs. 
This metric evaluates the MLLM’s response to the prompt ``Is [object] in this image?'' To emphasize that this is a binary VQA task, we appended the prompt with ``Please answer yes or no.''
To select objects referenced in the question prompt, we followed three different sampling options: random, popular, and adversarial. 
We evaluated performance across all options.

\textbf{TextHalu-Bench}~\cite{shu2025semantics} is a benchmark designed to evaluate
scene-text hallucination, with a particular focus on cases where semantic priors
can override faithful visual recognition.
It contains carefully curated samples spanning both \emph{semantic} and
\emph{non-semantic} text cases, and covers visually challenging scenarios such as
occlusions, low-contrast text, and uncommon fonts.
The benchmark includes two subtasks:
\emph{Spotting}, which requires direct text extraction from the image, and
\emph{Understanding}, which evaluates whether the recognized text is correctly
grounded for downstream question answering.
\textbf{Metrics:}
Following the official setup, we report the F1 score for the
\emph{Spotting} subtask (\texttt{S-F1}) and the F1 score for the
\emph{Understanding} subtask (\texttt{U-F1}).

\textbf{KIE-HVQA}~\cite{he2025seeing} is a benchmark for evaluating OCR
hallucination under degraded document conditions.
It focuses on visually-grounded question answering over key information extraction
(KIE) scenarios, with test samples drawn from real-world document types such as
identity cards and invoices, and augmented with simulated degradations including
blur, occlusion, and low contrast.
The benchmark is designed to measure whether a model can distinguish reliable
visual evidence from ambiguous or unreadable regions instead of over-relying on
linguistic priors.
\textbf{Metrics:}
Following the benchmark protocol, we report
answer accuracy on question-critical text regions (\texttt{Acc.}) and
text similarity (\texttt{Sim.}), where the latter reflects the normalized
string-level agreement between the extracted answer and the ground truth.

\textbf{MMBench}~\cite{liu2024mmbench} is a systematically designed
multiple-choice benchmark for holistic evaluation of multimodal models.
It covers a broad range of perception, reasoning, and knowledge abilities, and
provides both English and Chinese versions to support bilingual and
apples-to-apples comparisons across models.
\textbf{Metrics:}
We report the overall multiple-choice accuracy (\texttt{Acc.}).

\textbf{MMMU}~\cite{hendrycks2020measuring} is a college-level multimodal
benchmark designed to test expert-level perception, knowledge, and reasoning.
It contains about 11.5K multimodal questions collected from exams, quizzes, and
textbooks, spanning 30 subjects across six core disciplines, and involving highly
heterogeneous image types such as charts, diagrams, tables, maps, and scientific
figures.
\textbf{Metrics:}
We report the overall answer accuracy (\texttt{Acc.}).

\textbf{LLaVA-Wild}
(\textit{LLaVA-Bench in-the-Wild})
is a benchmark for evaluating general-purpose multimodal chat ability in open,
diverse real-world scenarios.
It contains 24 images and 60 manually curated questions covering daily-life
visual chat tasks, and groups them into conversation, detailed description, and
complex reasoning categories.
\textbf{Metrics:}
We are following the official evaluation protocol.

\newpage

\section{More Results}\label{apx:res}

\textbf{Ablation on View Generation Strategies.} 
To validate our design choices for constructing the ordered visual triplet $v_w \succ v_a \succ v_l$, we conduct a detailed ablation study on the generation strategies for the stronger-evidence view ($v_w$) and the weaker-evidence view ($v_l$), as detailed in \autoref{tab:view_construction_ablation}. 

\begin{table}[h]
\centering \vspace{-1.5em}
  \caption{Ablation on visual view construction strategies. Experiments are conducted on LLaVA-NeXT. Our default choices (Warping + Stochastic Degradation) achieve the best trade-off.}\vspace{0.5em}
\label{tab:view_construction_ablation}
    \begin{tabular}{ll*{3}{>{\centering\arraybackslash}p{5em}}}
    \toprule
    \multicolumn{2}{c}{Construction Strategies} & \multirow{2}{*}{POPE$\uparrow$} & \multirow{2}{*}{MMHal$\uparrow$} & \multirow{2}{*}{OBJHal$\downarrow$} \\
    \cmidrule[0.5pt](lr){1-2}
    Stronger View ($v_w$) & Weaker View ($v_l$) & & & \\
    \midrule
    \multicolumn{2}{l}{\textit{Baseline (LLaVA-NeXT w/o OPPO)}} & 85.5 & 3.50 & 11.3 \\
    \midrule
    Cropping & Stochastic Deg. & 86.8 & 3.49 & 5.8 \\
    Red Circle (Prompting) & Stochastic Deg. & 86.5 & 3.46 & 6.5 \\
    Warping (Ours) & Gaussian Blur & 87.2 & 3.51 & 5.2 \\
    Warping (Ours) & Cutout (Masking) & 86.3 & 3.44 & 6.1 \\
    \rowcolor{gray!10} Warping (Ours) & Stochastic Deg. (Ours) & \textbf{87.9} & \textbf{3.54} & \textbf{4.3} \\
    \bottomrule
    \end{tabular}
  \vspace{-1.em}
\end{table}
For the \textbf{stronger-evidence view ($v_w$)}, we compare our attention-guided \textit{Warping} with two alternatives: \textit{Cropping} (extracting the bounding box of the attended region) and \textit{Visual Prompting} (drawing a red bounding box or circle around the target). While all methods improve over the baseline, Warping is markedly superior. \textit{Cropping} forcibly discards the global visual context, which is often necessary for resolving complex queries or spatial relationships. \textit{Visual Prompting}, on the other hand, introduces artificial pixel artifacts that shift the image distribution away from natural scenes, making the alignment less generalizable. Warping elegantly avoids both issues by smoothly magnifying local salience while retaining peripheral context.

For the \textbf{weaker-evidence view ($v_l$)}, we compare our \textit{Stochastic Degradation} against uniform \textit{Gaussian Blur} and hard \textit{Cutout} (masking the attended region with black pixels). \textit{Cutout} performs poorly because it completely removes the visual evidence. When the evidence is entirely absent, penalizing the model for utilizing language priors becomes an ill-posed objective, as guessing is its only remaining option. This breaks the continuous spectrum assumption of our preference chain. \textit{Gaussian Blur} applies a uniform low-pass filter, which preserves too much coarse semantic structure, making it insufficiently ``weak'' for OCR-heavy or fine-grained tasks. Our \textit{Stochastic Degradation} strikes the optimal balance by progressively corrupting the fine-grained query-relevant cues while preserving the overall image statistics, providing a smoother and more robust gradient for preference learning.

\textbf{Attention Necessity under Localization Quality.}
A remaining question is whether OPPO truly benefits from meaningful query-aware localization, or whether its gains mainly come from generic visual perturbation. Since our stronger-evidence view $v_w$ is constructed by amplifying query-relevant regions identified from cross-attention, a natural test is to ask whether OPPO becomes more effective when the underlying attention prior more accurately covers the answer-bearing region. To this end, we conduct a localization-quality study on a held-out ROI-annotated evaluation set. For each sample, we compute the ROI coverage of the base model’s attention map, defined as the fraction of normalized attention mass falling inside the ground-truth answer-bearing region, and partition the samples into three bins of equal size: Low, Mid, and High coverage. We then compare the downstream performance of DPO, OPPO$\dagger$, and OPPO under the same matched-backbone setting. If OPPO mainly works because of query-aware evidence pickup, its advantage over DPO and random-warp should increase with localization quality, rather than remain constant across bins. This analysis complements the robustness result in Table \ref{tab:localization_quality_oppo}, instead of only testing whether OPPO collapses under noisy priors, we ask whether better localization systematically amplifies the gains of evidence-aware preference optimization. We observe a consistent monotonic trend: the gain of OPPO over DPO increases from +1.1 in the low-coverage bin to +2.8 in the mid-coverage bin and +4.2 in the high-coverage bin. In contrast, the random-warp variant remains consistently weaker, especially when localization quality is high. This indicates that OPPO does not simply benefit from generic visual perturbation; rather, it becomes increasingly effective when the attention prior better captures the answer-bearing evidence. Notably, OPPO still improves over DPO even in the low-coverage regime, suggesting that localization quality is not a strict prerequisite for usefulness, but a key factor governing the size of the improvement.

\begin{table}[t]
\centering
\small \vspace{-1.5em}
\setlength{\tabcolsep}{3pt}
\caption{Localization quality vs. OPPO gains.
Samples are partitioned into three bins according to the base model’s attention coverage over the answer-bearing ROI. We report POPE accuracy under the matched-backbone LLaVA-NeXT setting. As localization quality improves, the gain of OPPO over both DPO and random-warp increases monotonically, indicating that OPPO benefits specifically from better query-aware evidence localization rather than from generic visual perturbation alone.}\vspace{0.5em}
\label{tab:localization_quality_oppo}
\begin{tabular}{lcccccccc}
\toprule
Coverage bin & Coverage & Base Acc. & DPO Acc. & OPPO$\dagger$ Acc. & OPPO Acc. & $\Delta$(vs. DPO) & $\Delta$(vs. OPPO$\dagger$) \\
\midrule
Low     & 0.18   & 84.6 & 83.8 & 84.5 & 84.9 & +1.1 & +0.4 \\
Mid     & 0.47 & 85.4 & 85.0 & 86.3 & 87.8 & +2.8 & +1.5 \\
High    & 0.79   & 86.5 & 86.8 & 88.4 & 91.0 & +4.2 & +2.6 \\
\midrule
Overall & 0.48 & 85.5 & 85.2 & 86.4 & 87.9 & +2.7 & +1.5 \\
\bottomrule
\end{tabular}
\end{table}

\begin{wraptable}{r}{0.45\textwidth}
\centering
\caption{Comparison of Error Counts between Base MLLM and AttWarp across Different Categories (\# Error $\downarrow$)}
\label{tab:error_analysis}
    \resizebox{0.45\textwidth}{!}
    {
\begin{tabular}{lcc}
\toprule
\textbf{Category} & \textbf{Base} & \textbf{OPPO} \\ \midrule
Fine-Grained Details    & 12 & 3  \\
Hallucination           & 11 & 7  \\
Misaligned Attention    & 4  & 3 \\
Size                    & 18 & 17 \\
Semantically Correct    & 8  & 6  \\
Compositional Reasoning & 7  & 4  \\ \bottomrule
\end{tabular}}
\end{wraptable}
\textbf{Error Analysis.} Following \cite{dalal2025constructive}, we analyzed 150 randomly sampled VQA tasks from the GQA and TextVQA evaluation sets, identifying 60 errors for the baseline LLaVA model and 39 for OPPO. These errors were classified into six failure modes: (1) Fine-Grained Details (visually minute targets), (2) Hallucination (fabricating non-existent details), (3) Misaligned Attention (focusing on incorrect objects), (4) Size (object scale misjudgments), (5) Semantically Correct (accurate but paraphrased answers), and (6) Compositional Reasoning (complex multi-object relationships). Tab. \ref{tab:error_analysis} demonstrates that OPPO yields notably fewer errors in fine-grained and compositional tasks. Nevertheless, warping can sometimes obscure the peripheral context necessary for global reasoning, and noisy underlying attention may lead to performance degradation. Finally, because warping preserves relative proportions despite altering absolute sizes, errors in size-related tasks are effectively limited.

\section{Implement Details}\label{apx:implement}
\subsection{Attention-Guided Stronger-View Construction: AttCrop and AttWarp}
\label{app:attcrop_attwarp}

Given an image-question pair $(v_a, x)$, we first extract a query-conditioned
attention score matrix
\begin{equation}
A \in \mathbb{R}_{\ge 0}^{H \times W},
\end{equation}
where $A_{ij}$ measures the relevance of spatial location $(i,j)$ to the query.
In practice, $A$ is obtained by aggregating cross-modal attention weights from
the base MLLM and then upsampling them to the image resolution.\footnote{For
AttWarp, this follows the attention-guided image warping pipeline of
Dalal et al.~\cite{dalal2025constructive}.}

For numerical stability, we normalize $A$ into a probability map:
\begin{equation}
\widetilde{A}_{ij}
=
\frac{A_{ij}}{\sum_{u=1}^{H}\sum_{v=1}^{W} A_{uv} + \varepsilon},
\qquad
\sum_{i=1}^{H}\sum_{j=1}^{W}\widetilde{A}_{ij}=1,
\label{eq:att_norm}
\end{equation}
where $\varepsilon>0$ is a small constant.

\paragraph{Marginal attention profiles.}
We further compute the row- and column-wise marginal attention distributions:
\begin{equation}
p_y(i)=\sum_{j=1}^{W}\widetilde{A}_{ij},
\qquad
p_x(j)=\sum_{i=1}^{H}\widetilde{A}_{ij},
\label{eq:marginals}
\end{equation}
where $\sum_{i=1}^{H}p_y(i)=1$ and $\sum_{j=1}^{W}p_x(j)=1$.

\subsubsection{AttCrop}
\label{app:attcrop}

AttCrop constructs a stronger-evidence view by cropping the attended region and
resizing it back to the original input resolution. To obtain a deterministic
attention-guided crop, we first form the cumulative distributions
\begin{equation}
F_y(i)=\sum_{u=1}^{i} p_y(u),
\qquad
F_x(j)=\sum_{v=1}^{j} p_x(v).
\label{eq:crop_cdf}
\end{equation}
Given a mass-retention hyperparameter $\rho \in (0,1)$, we define the crop box
as the smallest axis-aligned rectangle that preserves the central $\rho$ mass
along both axes:
\begin{align}
i_{\min} &= \min \left\{ i : F_y(i) \ge \frac{1-\rho}{2} \right\},
&
i_{\max} &= \max \left\{ i : F_y(i) \le 1-\frac{1-\rho}{2} \right\}, \\
j_{\min} &= \min \left\{ j : F_x(j) \ge \frac{1-\rho}{2} \right\},
&
j_{\max} &= \max \left\{ j : F_x(j) \le 1-\frac{1-\rho}{2} \right\}.
\label{eq:crop_box}
\end{align}
The cropped image is then
\begin{equation}
v_{\mathrm{crop}}
=
v_a[i_{\min}\!:\!i_{\max},\, j_{\min}\!:\!j_{\max}],
\label{eq:crop_region}
\end{equation}
and the final AttCrop view is obtained by resizing the crop back to the input
resolution:
\begin{equation}
v_w^{\mathrm{crop}}
=
\mathrm{Resize}\!\left(v_{\mathrm{crop}},\, H,\, W\right).
\label{eq:attcrop_final}
\end{equation}

\noindent
In words, AttCrop enlarges the most attended region by discarding peripheral
content and reallocating all available pixels to the retained box.

\subsubsection{AttWarp}
\label{app:attwarp}

Unlike AttCrop, AttWarp preserves the full image content and redistributes
spatial resolution according to the query-conditioned attention map. Following
Dalal et al.~\cite{dalal2025constructive}, we perform \emph{rectilinear}
warping based on the row and column marginals in Eq.~\eqref{eq:marginals}.

We first convert the marginals into cumulative distribution functions:
\begin{equation}
C_y(i)=\sum_{u=1}^{i} p_y(u),
\qquad
C_x(j)=\sum_{v=1}^{j} p_x(v).
\label{eq:warp_cdf}
\end{equation}
Since $C_y$ and $C_x$ are monotone non-decreasing, they define valid
1D transport maps. We denote their inverse CDFs by
\begin{equation}
T_y = C_y^{-1},
\qquad
T_x = C_x^{-1}.
\label{eq:inverse_cdf}
\end{equation}

To express the warp continuously, let $(u,v)\in[0,1]^2$ be normalized
coordinates on the target warped image. The source coordinates are defined as
\begin{equation}
\phi_y(u)=T_y(u),
\qquad
\phi_x(v)=T_x(v),
\label{eq:warp_maps}
\end{equation}
where $T_y(u)$ and $T_x(v)$ are interpreted via linear interpolation between
discrete grid locations. The warped image is then produced by bilinear sampling:
\begin{equation}
v_w^{\mathrm{warp}}(u,v,c)
=
v_a\!\left(\phi_y(u),\, \phi_x(v),\, c\right),
\qquad c\in\{1,2,3\}.
\label{eq:bilinear_sampling}
\end{equation}

\noindent
Intuitively, regions with larger marginal attention mass occupy a larger span
in the warped image, while low-attention regions are compressed. Because the
mapping is rectilinear and defined separately along the vertical and horizontal
axes, AttWarp preserves the regular grid structure required by standard vision
encoders while keeping all original image content.

Both AttCrop and AttWarp use the same query-conditioned attention prior
$A$ but differ in how they allocate visual budget. AttCrop strengthens evidence
by \emph{discarding} low-attention context and zooming into the attended box,
whereas AttWarp strengthens evidence by \emph{redistributing} pixel density
toward attended regions while preserving the full scene. In our experiments,
AttWarp generally yields a better trade-off as it magnifies local evidence
without removing potentially useful global context.

\begin{figure}[h]
    \centering \vspace{-0.5em}
    \includegraphics[width=1\linewidth]{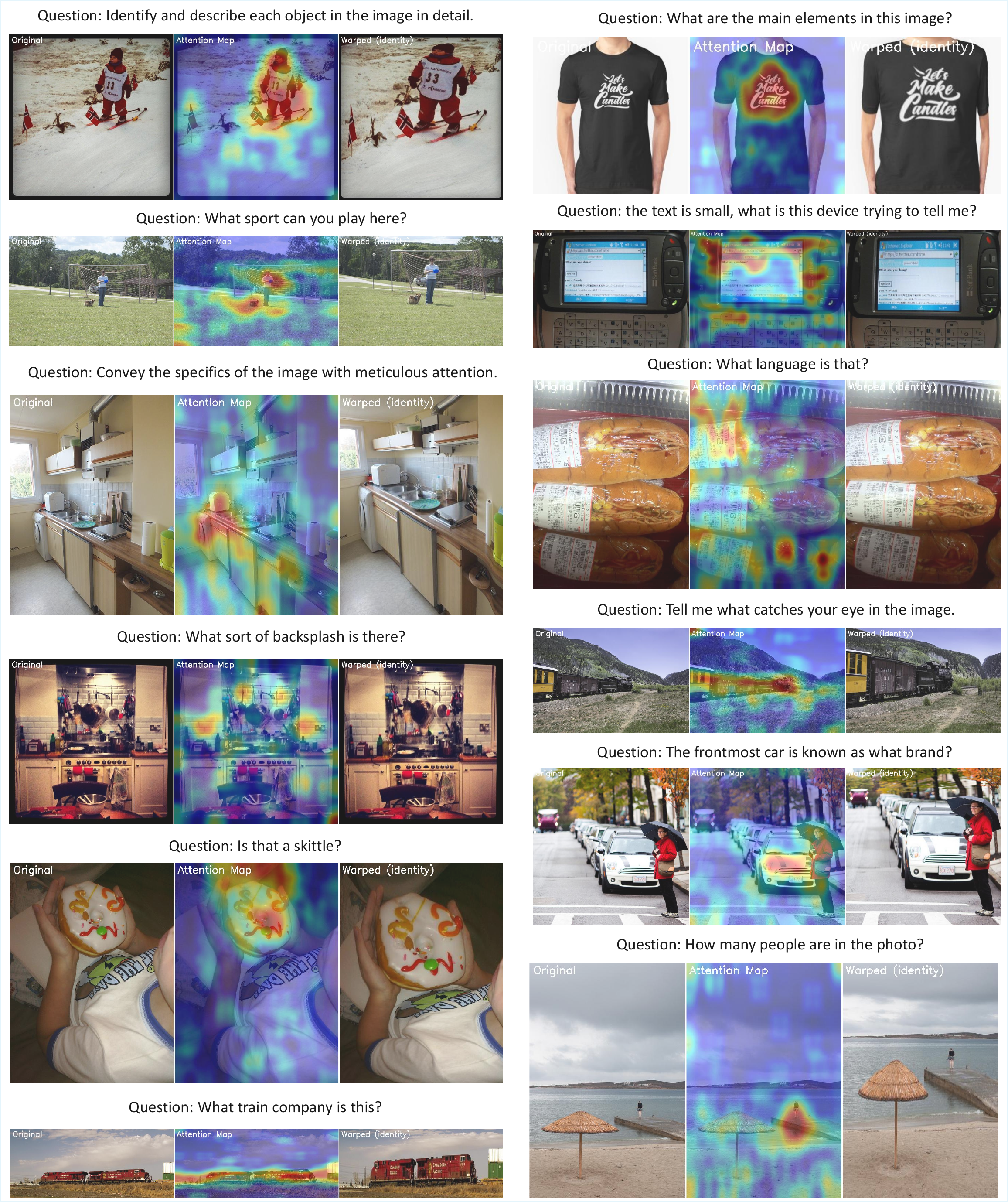}
    \caption{Qualitative results of AttWarp. We visualize the attention maps and the corresponding warped images across various VQA tasks. AttWarp effectively aligns the model's visual focus with task-relevant regions (\textit{e.g.}, the Skittle on the donut or the train logo).}
    \label{fig:placeholder1}\vspace{-0.75em}
\end{figure}

\section{Impact Statements}\vspace{-0.5em}
This work improves the visual groundedness of multimodal large language models and may benefit applications such as document understanding, visual question answering, and other settings where reducing hallucination is important for reliability. By aligning model preference with visual evidence strength, our method can help make model outputs more faithful to the input image.

However, stronger fine-grained visual grounding and OCR ability may also introduce dual-use risks, such as sensitive text extraction or surveillance-related misuse. Moreover, reduced hallucination does not imply error-free behavior, and deployment in high-stakes settings should therefore remain cautious. We encourage responsible use with appropriate safeguards, domain-specific evaluation, and clear restrictions on misuse.

\newpage



\newpage

\end{document}